\documentclass[letterpaper]{article} 
\usepackage[]{aaai25}  
\usepackage{times}  
\usepackage{helvet}  
\usepackage{courier}  
\usepackage[hyphens]{url}  
\usepackage{graphicx} 
\usepackage{natbib}  
\usepackage{caption} 
\usepackage{algorithm}
\usepackage{algorithmic}

\usepackage{newfloat}
\usepackage{listings}
\DeclareCaptionStyle{ruled}{labelfont=normalfont,labelsep=colon,strut=off} 
\floatstyle{ruled}
\newfloat{listing}{tb}{lst}{}
\floatname{listing}{Listing}
\title{Partial Identifiability in Inverse Reinforcement Learning\\ For Agents With Non-Exponential Discounting}
\author {
    Joar Skalse, 
    Alessandro Abate 
}
\affiliations {
    Department of Computer Science, University of Oxford\\
    joar.skalse@cs.ox.ac.uk, aabate@cs.ox.ac.uk
}

\usepackage{caption, amsmath, amssymb, mathtools, extpfeil, float, yfonts, csquotes, lipsum, fancyhdr, algorithm, algorithmic, amsthm, tikz, xcolor}

\theoremstyle{plain}
\newtheorem{theorem}{Theorem}
\newtheorem{proposition}{Proposition}
\newtheorem{lemma}{Lemma}
\theoremstyle{definition}
\newtheorem{definition}{Definition}
\theoremstyle{remark}

\newtheorem{example}{Example}

\usepackage{ragged2e}
\usetikzlibrary{automata, positioning}

\begin{document}

\newcommand{\J}{\mathcal{J}}

\newcommand{\Rspace}{{\hat{\mathcal{R}}}}

\newcommand{\Environment}{M}
\newcommand{\States}{\mathcal{S}}
\newcommand{\Actions}{\mathcal{A}}
\newcommand{\reward}{R}
\newcommand{\TransitionDistribution}{\tau}
\newcommand{\InitStateDistribution}{\mu_0} 
\newcommand{\discount}{\gamma}
\newcommand{\TransitionSet}{\mathcal{T}}
\newcommand{\InitStateSet}{\mathcal{I}}

\newcommand{\norm}[1]{\lVert#1\rVert}

\newcommand{\MDP}{\langle\States, \Actions, \{s_\top\}, \TransitionDistribution, \InitStateDistribution, 
\reward, \discount\rangle}
\newcommand{\MDPd}{\langle\States, \Actions, \{s_\top\}, \TransitionDistribution, \InitStateDistribution, 
\reward, d \rangle}
\newcommand{\ENV}{\langle \States, \Actions, \TransitionDistribution, \InitStateDistribution \rangle}
\newcommand{\ENVd}{\langle \States, \Actions, \TransitionDistribution, \InitStateDistribution, d \rangle}

\newcommand{\SxS}{{\States{\times}\States}}
\newcommand{\SxA}{{\States{\times}\Actions}}
\newcommand{\SxAxS}{{\States{\times}\Actions{\times}\States}}

\newcommand{\M}{\mathcal{M}}
\newcommand{\Q}{Q}
\newcommand{\V}{V}
\newcommand{\A}{A}
\newcommand{\Return}{G}
\newcommand{\Evaluation}{\mathcal{J}}
\newcommand{\EvaluationMCE}{\Evaluation^{\mathrm{H}}_\beta}
\newcommand{\Qfor}[1]{\Q^{#1}}
\newcommand{\Vfor}[1]{\V^{#1}}
\newcommand{\Afor}[1]{\A^{#1}}
\newcommand{\QStar}{\Q^\star}
\newcommand{\VStar}{\V^\star}
\newcommand{\AStar}{\A^\star}
\newcommand{\QSoft}{\Q^{\mathrm{H}}_\alpha}
\newcommand{\QSoftN}[1]{\Q^{\mathrm{H}}_{\beta,#1}}

\newcommand{\policy}{\pi}
\newcommand{\OptimalPolicy}{\policy_\star}
\newcommand{\basePolicy}{\policy_0}
\newcommand{\greedyPolicyWrt}[1]{\policy^{#1}_{\star}}
\newcommand{\epsGreedyPolicyWrt}[1]{\policy^{#1}_{\epsilon}}
\newcommand{\BoltzmannPolicyWrt}[1]{\policy^{#1}_{\beta}}
\newcommand{\BoltzmannRationalPolicy}{\policy^\star_\beta}
\newcommand{\MaximumEntropyPolicy}{\policy_\beta}
\newcommand{\MCEPolicy}{\policy^{\mathrm{H}}_\beta}

\newcommand{\Expect}[2]{\mathbb{E}_{#1}\left[{#2}\right]}

\newcommand{\red}[1]{\textcolor{red}{#1}}
\newcommand{\green}[1]{\textcolor{teal}{#1}}
\newcommand{\orange}[1]{\textcolor{orange}{#1}}

\newcommand{\PS}{\mathrm{PS}_\gamma}
\newcommand{\SR}{S'\mathrm{R}_\TransitionDistribution}
\newcommand{\LS}{\mathrm{LS}}
\newcommand{\OPT}{\mathrm{OP}_{\TransitionDistribution,\gamma}}
\newcommand{\Mx}{\mathrm{M}_\mathcal{X}}
\newcommand{\CS}{\mathrm{CS}}

\newcommand{\kPS}[1]{\mathrm{PS}_{\discount,\InitStateDistribution}^{#1}}

\maketitle

\begin{abstract}
The aim of inverse reinforcement learning (IRL) is to infer an agent's \emph{preferences} from observing their \emph{behaviour}. Usually, preferences are modelled as a reward function, $R$, and behaviour is modelled as a policy, $\pi$. One of the central difficulties in IRL is that multiple preferences may lead to the same observed behaviour. That is, $R$ is typically underdetermined by $\pi$, which means that $R$ is only \emph{partially identifiable}. Recent work has characterised the extent of this partial identifiability for different types of agents, including \emph{optimal} and \emph{Boltzmann-rational} agents. However, work so far has only considered agents that discount future reward \emph{exponentially}: this is a serious limitation, especially given that extensive work in the behavioural sciences suggests that humans are better modelled as discounting \emph{hyperbolically}. In this work, we newly characterise partial identifiability in IRL for agents with non-exponential discounting:  our results are in particular relevant for hyperbolical discounting, but they also more generally apply to agents that use other types of (non-exponential) discounting. We significantly show that generally IRL is unable to infer enough information about $R$ to identify the correct optimal policy, which entails that IRL alone can be insufficient to adequately characterise the preferences of such agents. 
\end{abstract}

\section{Introduction}

Inverse reinforcement learning (IRL) is a subfield of machine learning that aims to develop techniques for inferring an agent's \emph{preferences}  based on their \emph{actions}.
Preferences are typically modelled as a reward function, $R$, and behaviour is typically modelled as a policy, $\pi$.
An IRL algorithm must additionally employ a \emph{behavioural model} that describes how $\pi$ is computed from $R$: by inverting this model, an IRL algorithm can then deduce $R$ from $\pi$.

There are many motivations and applications  underpinning IRL. For example, it can be used in \emph{imitation learning} \citep[e.g.][]{imitation2017} or as a tool for \emph{preference elicitation} \citep[e.g.][]{CIRL}. In the former case it is not fundamentally important that the learnt reward function corresponds to the actual preferences of the observed agent, as long as it aids the imitation learning process. However, in the latter, it is instead fundamental that the learnt reward function captures the preferences of the observed agent as closely as possible. 
In this paper, we are primarily concerned with IRL in the context of preference elicitation,  namely in settings where IRL is used to learn a representation of the preferences of an actual human subject, based on information about how that human behaves in some environment, and where we wish for the learnt reward function to capture these preferences as faithfully as possible.



One of the central challenges in IRL is that a given sequence of actions typically can be explained by many different goals. 
That is, there may be multiple reward functions that would produce the same policy under a given behavioural model.
This means that the goals of an agent are ambiguous, or \emph{partially identifiable}, even in the limit of infinite data.
To clearly understand the impact of this partial identifiability, it is important that this ambiguity can be quantified and characterised.
The ambiguity of the reward function in turn depends on the behavioural model. 
For some behavioural models, the partial identifiability has been studied 
\citep{ng2000, dvijotham2010, cao2021, kim2021, skalse2022, schlaginhaufen2023identifiability, towardstheoreticalunderstandingofIRL}. 
However, this existing work has focused on a small number of behavioural models that are prevalent in the current IRL literature, whereas for other plausible or more general behavioural models, the issue of partial identifiability has largely not been studied.

One of the most important parts of a behavioural model is the choice of the \emph{discount function}.
In a sequential decision problem, different actions may lead the agent to receive more or less reward at different points in time.
In these cases, it is common to let the agent discount future reward, so that reward which will be received sooner is given greater weight than reward which will be received later.
Discounting can be done in many ways, but the
two most prominent forms of discounting are \emph{exponential discounting}, according to which reward received at time $t$ is given weight $\gamma^t$; and \emph{hyperbolic discounting}, according to which reward received at time $t$ is given weight $1/(1 + kt)$. Here $\gamma \in (0,1]$ and $k \in (0, \infty)$ are two parameters.\footnote{For a more in-depth overview of discounting, see for example \citet{discounting_review}.}
At the moment, most work on IRL assumes that the observed agent discounts exponentially.
However, extensive work in the behavioural sciences suggests that humans (and other animals) are better modelled as using hyperbolic discounting \citep[e.g., ][]{inconsistencyempirical,mazur1987adjusting, green1996exponential, againstnormativediscounting, discounting_review}. It is therefore a significant limitation that current IRL work exclusively employs behavioural models with exponential discounting. 


In this paper, we provide the first study of partial identifiability in IRL with non-exponential discounting. Specifically, we first introduce three  new behavioural models for agents with general discounting. 
We then study the partial identifiability of the reward function under these models, and provide both an exact characterisation and a comparison between models. Notably, we show that IRL algorithms are unable to infer enough information about $R$ to identify the correct optimal policy based on observations of an agent that discounts non-exponentially, which importantly suggests that IRL alone is insufficient to adequately characterise the preferences of such agents. All of our results apply to agents that use any general form of discounting, including the important hyperbolic discount rule. Our results thereby substantially extend the existing literature on partial identifiability in IRL, and have the potential to make it relevant to human decision making, since in particular hyperbolic discounting is thought to better fit human behaviour than exponential discounting. 

Our analysis is mathematical, rather than empirical, to ensure that our results are exact and general. Moreover, we focus on behavioural models, rather than specific IRL algorithms, because we want to characterise what information is contained in certain types of data, and thereby discover limitations that apply to all IRL algorithms in this problem setting. This makes our results broadly applicable.



\subsection{Related Work}

The issue of partial identifiability in IRL has been studied for many behavioural models. In particular, \citet{ng2000} study optimal policies with state-dependent reward functions, \citet{dvijotham2010} study regularised MDPs with a particular type of dynamics, \citet{cao2021} study how the reward ambiguity can be reduced by combining information from multiple environments, \citet{skalse2022} study three different behavioural models and introduce a framework for reasoning about partial identifiability in reward learning, 
\citet{schlaginhaufen2023identifiability} study ambiguity in constrained MDPs, and \citet{towardstheoreticalunderstandingofIRL} quantify sample complexities for optimal policies. However, all these papers assume exponential discounting.

Most IRL algorithms are designed for agents that discount exponentially, but some papers have considered hyperbolic discounting \citep{IgnorantAndInconsistent, IrrationalityCanHelp, schultheis2022reinforcement}. However, these papers do not formally characterise the identifiability of $R$ given their algorithms.

\section{Preliminaries}

In this section, we give a brief overview of all material that is required to understand this paper, together with our basic assumptions, and our choice of terminology.

\subsection{Reinforcement Learning}


In this paper, we take a \emph{Markov decision processes} (MDP) to be a tuple
$\MDP$
where
  $\States$ is a set of \emph{states},
  $\Actions$ is a set of \emph{actions},
  $\{s_\top\}$ is a \emph{terminal state},
  $\TransitionDistribution : \SxA \to \Delta(\States \cup \{s_\top\})$ is a \emph{transition function},
  $\InitStateDistribution \in \Delta(\States)$ is an \emph{initial state
  distribution}, 
  $\reward : \States \times \Actions \times (\States \cup \{s_\top\}) \to \mathbb{R}$ is a \emph{reward
    function}, 
    and $\discount \in (0, 1]$ is a \emph{discount rate}.
We will also assume that $\States$ and $\Actions$ are finite.
A \textit{policy} is a function $\policy : (\SxA)^\star \times \States \to \Delta(\Actions)$.
If a policy $\pi$ can be expressed as a simpler function $\States \to \Delta(\Actions)$, then we say that it is \emph{stationary}. 
We use $\mathcal{R}$ to denote the set of all reward functions definable over $\States$ and $\Actions$, i.e.\  $\mathbb{R}^{\SxAxS}$, and $\Pi$ to denote the set of all (stationary and non-stationary) policies that can be defined over $\States$ and $\Actions$, i.e.\ $\Delta(A)^{(\SxA)^\star\times \States}$.

A \emph{trajectory} $\xi = \langle s_0, a_0, s_1 \dots \rangle$ is a (finite or infinite) sequence of states and actions that form a path in an MDP.
If $s_\top \in \xi$, then we assume that $\xi$ is finite, and that $s_\top$ is the last state in $\xi$.\footnote{More precisely, this means that a trajectory is an element of $(\States \times \Actions)^\star \times (\States \cup \{s_\top\}) \cup (\States \times \Actions)^\omega$.}
The \emph{return function} $\Return$ gives the cumulative discounted reward of a trajectory, 
$\Return(\xi) = \sum_{t=0}^{|\xi|} \discount^t \reward(s_t, a_t, s_{t+1})$.
The \emph{value function} $\Vfor{\policy} : \States \rightarrow \mathbb{R}$ of a (stationary) policy $\pi$ encodes the expected cumulative discounted reward from each state under policy $\pi$, and its related $Q$-function 
is $\Qfor{\policy}(s,a) = \Expect{S' \sim \tau(s,a)}{R(s,a,S') + \discount \Vfor\policy(S')}$.
If a policy $\pi$ satisfies that $V^\pi(s) \geq V^{\pi'}(s)$ for all states $s$ and all policies $\pi'$, then we say that $\pi$ is an \emph{optimal policy}.
$\QStar$ denotes the $Q$-function of optimal policies. This function is unique, even when there are multiple optimal policies.

We say that an MDP is \emph{episodic} if there is some $H \in \mathbb{N}$ such that any policy with probability $1$ will enter the terminal state $s_\top$ after at most $H$ steps, starting from any state. Note that $H \leq |\States|$ in any episodic MDP. We say that an MDP is \emph{non-episodic} if $s_\top$ is unreachable from any $s \in \States$. Note that an MDP may be neither episodic or non-episodic.
Since the transition function $\TransitionDistribution$ alone determines whether or not an MDP is episodic, non-episodic, or otherwise, we will also refer to episodic and non-episodic transition functions. In episodic MDPs, we refer to a trajectory that starts in some state $s_0 \in \mathrm{supp}(\InitStateDistribution)$ and ends in $s_\top$ as an \emph{episode}.


When constructing examples of MDPs, it will sometimes be convenient to let the set of actions $\Actions$ vary between different states. In these cases, we may assume that each state has a \enquote{default action} that is chosen from the actions available in that state, and that all actions that are unavailable in that state simply are equivalent to the default action. 



\subsection{Inverse Reinforcement Learning}


In IRL we wish to infer a reward function $R$ based on a policy $\pi$ that has been computed from $R$. To do this, we need a \emph{behavioural model} that describes how $\pi$ relates to $R$. One of the most common models is known as \emph{Boltzmann Rationality} \citep[e.g.\ ][]{ramachandran2007}, and is given by
$\mathbb{P}(\pi(s) = a) \propto \exp{\beta Q^\star(s,a)}$,
where $\beta$ is a temperature parameter, and $Q^\star$ is the optimal $Q$-function for exponential discounting of $R$ with fixed discount parameter $\gamma$. 
In other words, a Boltzmann-rational policy is given by applying a \emph{softmax function} to $Q^\star$.
%
An IRL algorithm infers $R$ from $\pi$ by inverting a behavioural model. There are many algorithms for doing this \citep[e.g.][and many others]{ng2000, ramachandran2007, haarnoja2017}, but for the purposes of this paper, it will not be important to be familiar with the details of  these algorithms. 


\subsection{Partial Identifiability}

%

Following \citet{skalse2022}, we will characterise partial identifiability in terms of transformations and equivalence relations on $\mathcal{R}$. Let us first introduce a number of definitions: 

\begin{definition}\label{def:behavioural_model}
A \emph{behavioural model} is a function $\mathcal{R} \to \Pi$.
\end{definition}

For example, we can consider the function $b_{\beta,\tau,\gamma}$ that, given a reward $R$, returns the Boltzmann-rational policy with temperature $\beta$ in the MDP $\MDP$. Note that we consider the environment dynamics (i.e.\ the transition function, $\tau$) to be part of the behavioural model. This makes it easier to reason about if and to what extent the identifiability of $R$ depends on $\tau$.

\begin{definition}\label{def:ambiguity}
A \emph{reward transformation} is a function $t : \mathcal{R} \to \mathcal{R}$.
Given a behavioural model $f : \mathcal{R} \to \Pi$ and a set $T$ of reward transformations, we say that $f$ \emph{determines $R$ up to $T$} if $f(R_1) = f(R_2)$ if and only if $R_2 = t(R_1)$ for some $t \in T$.
\end{definition}

Definition~\ref{def:ambiguity} states that the partial identifiability of the reward under a particular behavioural model can be fully characterised in terms of reward transformations.
To see this, let us first build an abstract model of an IRL algorithm.
Let $R^\star$ be the true reward function.
We model the data source as a function $f : \mathcal{R} \to \Pi$,
so that the learning algorithm observes the policy $f(R^\star)$.
A reasonable learning algorithm should learn (or converge to) a reward function $R_H$ that is compatible with the observed policy, i.e.\ a reward such that $f(R_H) = f(R^\star)$.
This means that if $f$ determines $R$ up to $T$, then an IRL algorithm based on $f$ is unable to distinguish between two reward functions $R_1$, $R_2$ exactly when $R_1$ and $R_2$ are related by some transformation in $T$. Hence, the partial identifiability of $R$ under $f$ can be characterised \cite{skalse2022}.

\section{The Non-Exponential Setting}

In order to study partial identifiability in IRL with non-exponential discounting, we must first develop behavioural models for this setting. Since the Boltzmann-rational model is the most prominent behavioural model in the standard (exponentially discounted) setting, we will generalise the Boltzmann-rational behavioural model to work for general discount functions.
However, before we can do this, we must first generalise the basic RL setting.
We will allow a \emph{discount function} to be any function $d : \mathbb{N} \to [0,1]$ such that $d(0) = 1$. Some noteworthy examples of discount functions include \emph{exponential discounting}, where $d(t) = \gamma^t$, \emph{hyperbolic discounting}, where $d(t) = 1/(1+k \cdot t)$, and \emph{bounded planning}, where $d(t) = 1 \text{ if } t \leq n, \text{ else } 0$.
Here $\gamma$, $k$, and $n$ are parameters. 
In this paper, we are especially interested in hyperbolic discounting, since it is argued to be a good match to human behaviour. However, most of our results apply to arbitrary discount functions.\footnote{Note that average-reward reinforcement learning \citep{Mahadevan1996} is not covered by this setting. Averaging the rewards is not a form of discounting, but is instead an alternative to discounting. However, also note that if all possible episodes have the same length, then the average-reward objective is equivalent to using a constant discount function.}



Many of the basic definitions in RL can straightforwardly be extended to general discount functions. We consider an MDP to be a tuple $\MDPd$, where $d$ may be any discount function. As usual, we define the trajectory return function as $G(\xi) = \sum_{t=0}^{|\xi|} d(t) \cdot R(\xi_t)$. We say that $V^\pi(\xi)$ is the expected future discounted reward if you start at the (finite) trajectory $\xi$ and sample actions from $\pi$, and that $Q^\pi(\xi,a)$ is the expected future discounted reward if you start at trajectory $\xi$, take action $a$, and then sample all subsequent actions from $\pi$.\footnote{As we will soon see, we will have to consider non-stationary policies in the setting with non-exponential discounting. This is why the $Q$-function and value function must be parameterised by the entire past trajectory, instead of just the current state.} 
As usual, if $\pi$ is stationary, then we let $V^\pi$ and $Q^\pi$ be parameterised by the current state, instead of the past trajectory.

It will be convenient to also use value- and $Q$-functions that start discounting from a different time than zero --- we will indicate this with a superscript. Specifically, we let $V^{\pi,n}(\xi) = \mathbb{E}\left[\sum_{t=0}^\infty d(t+n) \cdot R(\zeta_t)\right]$, where the expectation is over a trajectory $\zeta$ given by starting with the (finite) trajectory $\xi$, and then sampling all subsequent actions from $\pi$. We also define $Q^{\pi,n}$ analogously. Note that $V^\pi = V^{\pi,0}$ and $Q^\pi = Q^{\pi,0}$. Intuitively speaking, $V^{\pi,n}(\xi)$ is the expected future discounted reward if you start at the (finite) trajectory $\xi$ and then sample all subsequent actions from $\pi$, but discount as though you are starting at time $n$. Note that with exponential discounting, we have that $V^\pi \propto V^{\pi,n}$ for all $n$, but not with non-exponential discounting.




For exponential discounting where $\gamma < 1$, we have that $\sum_{t=0}^\infty \gamma^t < \infty$. 
This ensures that $V^\pi$ always is strictly finite for any choice of $R$ and $\tau$. However, if $\sum_{t=0}^\infty d(t)$ diverges, then $V^\pi$ will also diverge for some $R$ and $\tau$, which of course is problematic   for policy selection. 
Therefore, it could be reasonable to impose the requirement that $\sum_{t=0}^\infty d(t) < \infty$ as a condition on $d$.
Unfortunately, this would rule out the hyperbolic discount function. Since this discount function is of particular interest to us, 
we will instead impose conditions on the environment.
In particular, if the MDP is \emph{episodic} then $V^\pi$ is always finite, regardless of which discount function is chosen. For this reason, most of our results will assume that the environment is episodic.

An important property of general discount functions is that they can lead to preferences that are \emph{inconsistent over time}. To understand this, consider the following example: 
\begin{example}\label{example:gym_mdp}
Let \texttt{Gym} be the MDP where $\States = \{s_0, s_1, s_2\}$, $\Actions = \{\text{buy},\text{exercise},\text{enjoy},\text{go home}\}$, $\mu_0 = s_0$, and the transition function $\tau$ is the deterministic function given by the following labelled graph:
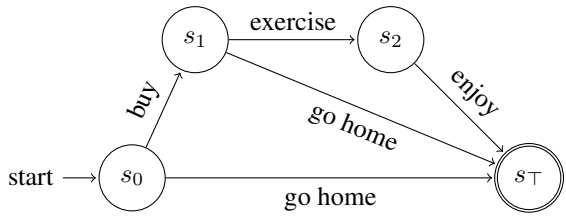
\begin{center}
\begin{tikzpicture}[shorten >=1pt,node distance=2.6cm,on grid,auto]
   \node[state, initial] (s_0)   {$s_0$}; 
   \node[state]         (s_1) [above right=of s_0, xshift=-1cm] {$s_1$};
   \node[state]         (s_2) [right=of s_1] {$s_2$};
   \node[state, accepting]         (s_3) [below right=of s_2] {$s_\top$};
    \path[->] 
    (s_0) edge [sloped] node {buy} (s_1)
          edge [swap] node {go home} (s_3)
    (s_1) edge [] node {exercise} (s_2)
          edge [sloped, swap] node {go home} (s_3)
    (s_2) edge [sloped] node {enjoy} (s_3)
    ;
\end{tikzpicture}
\end{center}
The discount function $d$ is the hyperbolic discount function, $d(t) = 1/(1+t)$, and $R$ is the reward function given by $R(buy) = -1$, $R(exercise) = -16$, $R(enjoy) = 30$, and $R(go\text{ }home) = 0$. \qed 
\end{example}


%
This is a deterministic, episodic environment with three states $\{s_0, s_1, s_2\}$, where $s_0$ is initial. 
In state $s_0$, the agent can choose between either buying a gym membership, or going home. If it buys the gym membership, then it gets to choose between exercising at the gym, or going home. If it exercises, then it gets to enjoy the benefits of exercise, after which the episode ends. Similarly, if the agent ever goes home, the episode also ends. 

%
%

We can calculate the value of each trajectory from the initial state $s_0$; $G(\text{go home}) = 0$, $G(\text{buy, go home}) = -1$, and $G(\text{buy, exercise, enjoy}) = 1$. This means that the most valuable trajectory from $s_0$ involves buying a gym membership, and then exercising. However, if we calculate the value of each trajectory from state $s_1$, we (paradoxically) find that $G(\text{go home}) = 0$ and $G(\text{exercise, enjoy}) = -1$. This means that the agent at state $s_0$ would prefer to buy a gym membership, and then exercising. However, after having bought the gym membership, the agent now prefers to go home instead of exercising.
In other words, the agent has preferences that are inconsistent over time. We can formalise this as follows. 

\begin{definition}\label{def:temporal_consistency}
A discount function $d$ is \emph{temporally consistent} if for all sequences $\{x_t\}_{t=0}^\infty, \{y_t\}_{t=0}^\infty$, 
$\sum_{t=0}^\infty d(t) \cdot x_t < \sum_{t=0}^\infty d(t) \cdot y_t$
implies that
$\sum_{t=0}^\infty d(t+n) \cdot x_t < \sum_{t=0}^\infty d(t+n) \cdot y_t$
for all $n$.
\end{definition}


Intuitively, if a discount function $d$ is temporally consistent, and at some time $n$ it prefers a sequence of rewards $\{x_t\}_{t=0}^\infty$ over another sequence $\{y_t\}_{t=0}^\infty$, then this is also true at every other time $n$. On the other hand, if $d$ is \emph{not} temporally consistent, then it may change its preference as time passes, as in the MDP from Example~\ref{example:gym_mdp}. 
It is easy to show that exponential discounting \emph{is} temporally consistent, and Example~\ref{example:gym_mdp} demonstrates that hyperbolic discounting is \emph{not} temporally consistent. 
What about other discount functions?
As it turns out, exponential discounting is the \emph{only} form of discounting that is temporally consistent.
This means that \emph{all other discount functions} can lead to preferences that are not consistent over time. 

\begin{proposition}\label{prop:temporal_consistency_exponential}
$d$ is temporally consistent if and only if $d(t) \propto \gamma^t$ for some $\gamma \in [0,1]$. 
\end{proposition}

For a proof of Proposition~\ref{prop:temporal_consistency_exponential}, see \citet{temporal_inconsistency_1} or \citet{temporal_inconsistency_2}. 
Note that this temporal inconsistency is an important reason for why hyperbolic discounting is considered to be a good fit for human data --- under experimental conditions, humans can exhibit \emph{preference reversals} in a way that is consistent with hyperbolic discounting \citep[see e.g.\ ][]{discounting_review}.
However, temporal inconsistency also implies that there no longer is an unamibiguous notion of what it means for a policy to be \enquote{better} than another policy in this setting.
For instance, in Example~\ref{example:gym_mdp}, should the \enquote{best} policy choose to exercise at $s_1$, or should it choose to go home? 
There are multiple ways to answer this question, which in turn means that there are multiple ways to formalise what it means for an agent to \enquote{use} hyperbolic discounting (or other non-exponential discount functions).
As such, modelling agents that discount non-exponentially (such as humans) involves some subtle modelling choices that are not present for exponentially discounting agents.
In the next section, we explore several ways of dealing with this issue. 

\subsection{New Behavioural Models for General Discounting}

We wish to construct behavioural models that are analogous to Boltzmann-rationality for non-exponential discounting. Recall that in the exponentially discounted setting, the Bolzmann-rational policy is given by applying a softmax function to the optimal $Q$-function. 
We must therefore first decide what it means for a policy to be \enquote{optimal} in this setting.
Because of temporal inconsistency, the ordinary notion of optimality does not automatically apply, and 
there are multiple ways to extend the concept.
Accordingly, we introduce three new definitions:



\begin{definition}\label{def:resolute_policy}
    A policy $\pi$ is 
    is \emph{resolute} if there is no $\pi'$ or $\xi$ such that $V^{\pi,|\xi|}(\xi) < V^{\pi',|\xi|}(\xi)$.
\end{definition}

A resolute policy maximises expected reward as calculated from the initial state. In other words, it effectively ignores the fact that its preferences might be changing over time, and instead always sticks to the preferences that it had at the start. In Example~\ref{example:gym_mdp}, a resolute policy would buy a gym membership, and then exercise. 

\begin{definition}\label{def:naive_policy}
A policy $\pi$ is \emph{na\"ive} if for each trajectory $\xi$, if $a \in \mathrm{supp}(\pi(\xi))$, then there is a policy $\pi^\star$ such that $\pi^\star$ maximises $V^{\pi^\star,0}(\xi)$ and $a \in \mathrm{supp}(\pi^\star(\xi))$.
\end{definition}

A na\"ive policy ignores the fact that its preferences may not be temporally consistent. Rather, in each state, it computes a policy that is resolute from that state, and then takes an action that this policy would have taken, without taking into account that it may not actually follow this policy later. In Example~\ref{example:gym_mdp}, a na\"ive policy would buy a gym membership, but then go home without exercising.

\begin{definition}\label{def:sophisticated_policy}
    A policy $\pi$ is \emph{sophisticated} if $\mathrm{supp}(\pi(\xi)) \subseteq \mathrm{argmax} Q^{\pi,0}(\xi,a)$ for all trajectories $\xi$.
\end{definition}

A sophisticated policy is aware that its preferences are temporally inconsistent, and acts accordingly. Specifically, $\pi$ is sophisticated if it only takes actions that are optimal \emph{given that all subsequent actions are sampled from $\pi$}. In Example~\ref{example:gym_mdp}, a sophisticated policy would choose to not exercise in state $s_1$. Hence, in state $s_0$, it would realise that in $s_1$ it would go home, instead of exercising. Thus, in $s_0$ it prefers to go home over buying a gym membership and then going home, and it chooses to go home without buying a membership.

For consistency, if $d(t) = \gamma^t$ for some $\gamma \in (0,1)$, then Definitions~\ref{def:resolute_policy}-\ref{def:sophisticated_policy} reduce to optimality. Formally:

\begin{theorem}
In an MDP with exponential discounting, the following are equivalent:
(1) $\pi$ is optimal,
(2) $\pi$ is resolute,
(3) $\pi$ is na\"ive, and
(4) $\pi$ is sophisticated.
\end{theorem}

All proofs are provided in the appendix.
Notice that, while Definitions~\ref{def:resolute_policy}-\ref{def:sophisticated_policy} are all equivalent under exponential discounting, they can be quite different if other forms of discounting are used, as already exemplified by  Example~\ref{example:gym_mdp}. 
As such, each of these definitions give us a reasonable way to extend the notion of an \enquote{optimal} policy to the setting with general discount functions. 

We next focus on the issue of \emph{existence}, showing that \emph{each} of these three types of policies are guaranteed to exist in \emph{any episodic} MDP, regardless of what discount function is employed.  Moreover, in each case, we show that this still holds if we restrict our attention to \emph{deterministic} policies. We also show that na\"ive and sophisticated policies both are guaranteed to exist if we restrict our attention to \emph{stationary} policies, but that there are episodic MDPs in which there are no stationary resolute policies.

\begin{theorem}
In any episodic MDP, there exists a deterministic resolute policy.
\end{theorem}

\begin{theorem}
In any episodic MDP, there exists a stationary, deterministic, na\"ive policy.
\end{theorem}
\begin{theorem}\label{thm:soph_existence}
In any episodic MDP, there exists a stationary, deterministic sophisticated policy.
\end{theorem}

\begin{proposition}\label{prop:no_stationary_resolute}
There are episodic MDPs with no stationary resolute policies.
\end{proposition}

Note that Proposition~\ref{prop:no_stationary_resolute} 
is a consequence of the fact that non-exponential discounting can lead to preferences that are not temporally consistent, and that the agent may reach a given state at various time steps. Thus, the action taken by a resolute agent in that state may depend on the time at which it reaches that state. On the other hand, na\"ive and sophisticated agents have the same preferences regardless of what has happened in the past. Also note that a given reward function may allow for \emph{several} policies that are resolute, na\"ive, or sophisticated: for example, all policies are both resolute, na\"ive, and sophisticated for a trivial reward function that is $0$ everywhere. This lack of \emph{uniqueness} should not be surprising. 

To work out behavioural models that are analogous to Boltzmann-rationality, we must next develop analogies to the optimal $Q$-function for the non-exponential setting. For resolute and na\"ive policies, this is straightforward:

\begin{definition}
Given an episodic MDP, the \emph{resolute $Q$-function} $Q^\mathrm{R} : \States \times \mathbb{N} \times \Actions \to \mathbb{R}$ is defined as 
$$
Q^\mathrm{R}(s,t,a) := \mathbb{E}_{S' \sim \TransitionDistribution(s,a)}\left[R(s,a,S') + \max_{\pi}V^{\pi,t+1}(S')\right].
$$
\end{definition}

\begin{definition}
Given an episodic MDP, the \emph{na\"ive $Q$-function} $Q^\mathrm{N} : \SxA \to \mathbb{R}$ is defined as $$
Q^\mathrm{N}(s,a) := \mathbb{E}_{S' \sim \TransitionDistribution(s,a)}\left[R(s,a,S') + \max_{\pi}V^{\pi,1}(S')\right].
$$
\end{definition}

Note that the resolute $Q$-function must depend on the current time, since resolute policies may have to be non-stationary. 
Also note that $Q^\mathrm{N}(s,a) = Q^\mathrm{R}(s,0,a)$.
We next show that these $Q$-functions are guaranteed to exist and to be unique in any episodic MDP. This means that we can talk about \enquote{the} resolute $Q$-function and \enquote{the} na\"ive $Q$-function for each given (episodic) MDP:

\begin{proposition}\label{prop:resolute_Q}
In any episodic MDP, the resolute $Q$-function $Q^\mathrm{R}$ exists and is unique.
\end{proposition}
\begin{proposition}\label{prop:naive_Q}
In any episodic MDP, the na\"ive $Q$-function $Q^\mathrm{N}$ exists and is unique.
\end{proposition}

A policy $\pi$ is resolute if and only if it only takes actions that maximise  $Q^\mathrm{R}$, and na\"ive if and only if it only takes actions that maximise  $Q^\mathrm{N}$. These $Q$-functions thus provide the appropriate generalisations of the optimal $Q$-function, $Q^\star$, corresponding to resolute and na\"ive policies respectively.

For sophisticated policies, the situation is more complicated. This is a consequence of the following fact: 

\begin{proposition}\label{prop:sophisticated_is_complicated}
There are episodic MDPs $M$ with hyperbolic discounting and policies $\pi_1$, $\pi_2$ such that both $\pi_1$ and $\pi_2$ are sophisticated in $M$, but such that $Q^{\pi_1} \neq Q^{\pi_2}$, and such that the policy $\pi_3$ given by
$$
\mathbb{P}(\pi_3(s) = a) = \frac{\mathbb{P}(\pi_1(s) = a) + \mathbb{P}(\pi_2(s) = a)}{2}
$$
is not sophisticated in $M$.
\end{proposition}

This implies that there need not be a unique \enquote{sophisticated $Q$-function} in a given episodic MDP. Moreover, unlike resolute and na\"ive policies (and optimal policies in exponentially discounted MDPs), the set of sophisticated policies is not (in general) convex. This means that we cannot create a function $f : \SxA \to \mathbb{R}$ such that a policy is sophisticated if and only if it only takes actions that maximise $f$. Nonetheless, we can still create a reasonable \enquote{canonical} sophisticated $Q$-function, though this will require a bit more work:

\begin{definition}
In an episodic MDP, a stationary sophisticated policy $\pi$ is \emph{canonical} if in all $s$ we have that
$$
\mathbb{P}(\pi(s) = a_1) = \mathbb{P}(\pi(s) = a_2)
$$
for all $\{a_1, a_2\} \in \mathrm{argmax} Q^{\pi,0}(s,a)$. We the define the \emph{sophisticated $Q$-function} $Q^\mathrm{S}$ as 
$$
Q^\mathrm{S} = Q^{\pi,0}, 
$$
where $\pi$ is the canonical sophisticated policy, provided that this policy exists and is unique.
\end{definition}

Thus, a sophisticated policy is \enquote{canonical} if it is stationary, and if it mixes uniformly between all actions that have equal value. We next show that any episodic MDP always has a unique canonical sophisticated policy, which in turn means that $Q^\mathrm{S}$ is well-defined.\footnote{However, unlike what is the case for $Q^\mathrm{R}$ and $Q^\mathrm{N}$, it is not the case that a policy is sophisticated if and only if it only takes actions that maximise $Q^\mathrm{S}$. This is exemplified by the environment that is used in the proof of Proposition~\ref{prop:sophisticated_is_complicated}.} 

\begin{proposition}\label{prop:sophisticated_Q}
In any episodic MDP, the sophisticated $Q$-function $Q^\mathrm{S}$ exists and is unique.
\end{proposition}

Given the above results, we are now finally equipped to define three behavioural models that generalise Boltzmann-rationality to non-exponential discounting:

\begin{definition}
Given an episodic transition function $\tau$, discount $d$, and temperature $\beta \in (0,\infty)$, the \emph{Boltzmann-resolute} behavioural model is the function $r_{\tau,d,\beta} : \mathcal{R} \to \Pi$ for which $r_{\tau,d,\beta}(R)$ is the policy $\pi$ such that
$$
\mathbb{P}(\pi(\xi) = a) \propto \exp{\beta Q^\mathrm{R}(s, |\xi|, a)}, 
$$
where $s$ is the last state in $\xi$.
\end{definition}

\begin{definition}
Given an episodic transition function $\tau$, discount $d$, and temperature $\beta \in (0,\infty)$, the \emph{Boltzmann-na\"ive} behavioural model is the function $n_{\tau,d,\beta} : \mathcal{R} \to \Pi$ for which $r_{\tau,d,\beta}(R)$ is the stationary policy $\pi$ such that
$$
\mathbb{P}(\pi(s) = a) \propto \exp{\beta Q^\mathrm{N}(s, a)}.
$$
\end{definition}

\begin{definition}
Given an episodic transition function $\tau$, discount $d$, and temperature $\beta \in (0,\infty)$, the \emph{Boltzmann-sophisticated} behavioural model is the function $s_{\tau,d,\beta} : \mathcal{R} \to \Pi$ for which $r_{\tau,d,\beta}(R)$ is the stationary policy $\pi$ such that
$$
\mathbb{P}(\pi(s) = a) \propto \exp{\beta Q^\mathrm{S}(s, a)}.
$$
\end{definition}

\smallskip
In summary, we have devised three behavioural models that generalise the Boltzmann-rational one to general discount functions. It is of course an important question which of these behavioural models might provide the best fit to human behaviour: we consider this issue to be out of scope for this paper, and will instead analyse each of the behavioural models.



\section{Partial Identifiability for General Discounting}

Having laid the groundwork necessary to generalise the Boltzmann-rational behavioural model to non-exponential discounting, we are now able to characterise how ambiguous the reward function is for these new behavioural models. We will first provide an exact characterisation of this ambiguity, expressed in terms of necessary and sufficient conditions. 
Moreover, in order to interpret intuitively these results, we will also provide a number of results with more intuitive takeaways. 

\subsection{Exact Characterisation}

If two reward functions $R_1$, $R_2$ have the property that $Q^\star_2(s,a) = Q^\star_1(s,a) - \Phi(s)$ for some function $\Phi : \States \to \mathbb{R}$, where $Q^\star_1$ and $Q^\star_2$ are the optimal $Q$-functions for $R_1$ and $R_2$, then those reward functions are said to differ by \emph{potential shaping} \citep{ng1999}.\footnote{Note that the definition in \citet{ng1999} technically differs from this, but it can be shown to be equivalent.} It can be shown that Boltzmann-rational policies (with exponential discounting) determine $R$ up to potential shaping \citep{skalse2022}. We will show that this result generalises to the setting with non-exponential discounting, if the definition of potential shaping is adjusted appropriately. To state this result properly, we must first introduce two new definitions:

\begin{definition}
Given an episodic MDP, we say that two reward functions $R_1$, $R_2$ differ by \emph{sophisticated potential shaping} if there is a \emph{potential function} $\Phi : \States \to \mathbb{R}$ such that
$$
Q^\mathrm{S}_2(s,a) = Q^\mathrm{S}_1(s,a) - \Phi(s) 
$$
for all $s \in \States$ and $a \in \Actions$, where $Q^\mathrm{S}_1$ and $Q^\mathrm{S}_2$ are computed relative to $\TransitionDistribution$ and $d$ (for $R_1$ and $R_2$ respectively).
\end{definition}

\begin{definition}
Given an episodic MDP, we say that two reward functions $R_1$, $R_2$ differ by \emph{na\"ive potential shaping} if there is a \emph{potential function} $\Phi : \States \to \mathbb{R}$ such that
$$
Q^\mathrm{N}_2(s,a) = Q^\mathrm{N}_1(s,a) - \Phi(s) 
$$
for all $s \in \States$ and $a \in \Actions$, where $Q^\mathrm{N}_1$ and $Q^\mathrm{N}_2$ are computed relative to $\TransitionDistribution$ and $d$ (for $R_1$ and $R_2$ respectively).
\end{definition}

Note that 
these two definitions do not state if such reward functions actually exist, nor do they state how to compute them. Our next result therefore shows that we always can find a reward $R_2$ that differs from $R_1$ by na\"ive or sophisticated potential shaping with $\Phi$, for any $R_1$ and any $\Phi$. The key insight is that this $R_2$ can be computed from $R_1$ and $\Phi$ via backwards induction, provided that the MDP is episodic.

\begin{theorem}\label{thm:potential_shaping_exists}
For any episodic MDP 
with reward $R_1$
and for any potential function $\Phi : \States \to \mathbb{R}$, there exists a reward function $R_2$ that differs from $R_1$ by na\"ive potential shaping with $\Phi$, and a reward function $R_3$ that differs from $R_1$ by sophisticated potential shaping with $\Phi$.
\end{theorem}

Using this, we can now provide an exact characterisation of the ambiguity of the Boltzmann-sophisticated and the Boltzmann-na\"ive behavioural model, which is applicable for any discount function. This is a core result:

\begin{theorem}\label{thm:sophisticated_ambiguity}
In any episodic MDP, the Boltzmann-sophisticated policy determines $R$ up to sophisticated potential shaping.
\end{theorem}

\begin{theorem}\label{thm:naive_ambiguity}
In any episodic MDP, the Boltzmann-na\"ive policy determines $R$ up to na\"ive potential shaping.
\end{theorem}




It would be desirable to generalise these results to the Boltzmann-resolute behavioural model next.  However, this presents a number of challenges, primarily stemming from the fact that the Boltzmann-resolute policy may be time-dependent. We detail and discuss these issues in the appendix.

It is important to note that whether or not two reward functions $R_1$ and $R_2$ differ by na\"ive or sophisticated potential shaping is relative to a given transition function, and that there is no simple closed-form expression for computing $R_2$ based on $R_1$ and $\Phi$ \citep[unlike what is the case for potential shaping for exponential discounting, as introduced by][]{ng1999}. The reason for this is that the optimal $Q$-function $Q^\star$ in the exponentially discounted case can be expressed by a local recursive equation (namely the Bellman optimality equation), but this is in general not possible for $Q^\mathrm{N}$ and $Q^\mathrm{S}$ with non-exponential discounting.


\subsection{Qualitative Characterisation}

We have provided an exact characterisation of the ambiguity of the underlying reward $R$ given both na\"ive and sophisticated policies. However, these necessary and sufficient conditions can appear to be technically sophisticated. For this reason, we next provide a result that is easier to interpret qualitatively:

\begin{theorem}\label{thm:no_robustness_some_tau}
Let $f_{\TransitionDistribution,d,\beta}, g_{\TransitionDistribution,d,\beta} \in \{r_{\tau,d,\beta},n_{\tau,d,\beta},s_{\tau,d,\beta}\}$ be two behavioural models. Let $d_1$ and $d_2$ be any two discount functions, and let $\beta_1, \beta_2 \in (0,\infty)$ be any two temperature parameters. Then unless $d_1(t) = d_2(t)$ for all $t \leq |\States| - 1$, there exists an episodic transition function $\TransitionDistribution$ such that for any reward $R_1$ there exists a reward $R_2$ such that 
$$
f_{\TransitionDistribution,d_1,\beta_1}(R_1) = f_{\TransitionDistribution,d_1,\beta_1}(R_2),
$$ 
but such that 
$$
g_{\TransitionDistribution,d_2,\beta_2}(R_1) \neq g_{\TransitionDistribution,d_2,\beta_2}(R_2).
$$ 
Moreover, unless $d(t) = \alpha \cdot \gamma^t$ for some $\alpha, \gamma \in [0,1]$ and all $t \leq |\States| - 1$, we also have that, for any reward $R_1$ there exists a reward $R_2$, such that $f_{\TransitionDistribution,d_1,\beta_1}(R_1) = f_{\TransitionDistribution,d_1,\beta_1}(R_2)$, but such that $R_1$ and $R_2$ have different optimal policies under exponential discounting with $\gamma$.
\end{theorem}

Let us briefly unpack this result. Suppose $R_1$ is the true reward function, and that the training data for a given IRL algorithm is generated via $f_{\TransitionDistribution,d_1,\beta_1}$. Suppose also that we want to use the learnt reward function to compute the output of a different behavioural model $g_{\TransitionDistribution,d_2,\beta_2}$.
If $f_{\TransitionDistribution,d_1,\beta_1}(R_1) = f_{\TransitionDistribution,d_1,\beta_1}(R_2)$, then the IRL algorithm may converge to $R_2$ instead of $R_1$, since they have identical $f_{\TransitionDistribution,d_1,\beta_1}$-policies.
However, if $g_{\TransitionDistribution,d_2,\beta_2}(R_1) \neq g_{\TransitionDistribution,d_2,\beta_2}(R_2)$, then $R_1$ and $R_2$ have different policies under $g_{\TransitionDistribution,d_2,\beta_2}$. 
In other words, unless $d_1$ and $d_2$ are exactly equal over all time horizons that are possible in a given state space,\footnote{Note that in an episodic MDP, any episode has length at most $|\States| - 1$. In other words, the horizon cannot exceed $|\States| - 1$.} then the reward function is too ambiguous under $f_{\TransitionDistribution,d_1,\beta_1}$ to infer the correct value of $g_{\TransitionDistribution,d_2,\beta_1}$. For example, this means that the Boltzmann-sophisticated policy for the hyperbolic discount function, or the Boltzmann-na\"ive policy for the bounded planning discount function, both leave the underlying reward too ambiguous to infer the Boltzmann-rational policy under exponential discounting, and so on. This suggests that the ambiguity of the reward can be problematic if we want to use the learnt reward to compute a policy using a discount function that is different from that used by the observed agent.

Note that Theorem~\ref{thm:no_robustness_some_tau} says that there exists \emph{some} transition function $\tau$ for which this issue can occur.
This does, by itself, not rule out the possibility that the ambiguity of $R$ may be more modest for \enquote{typical} transition functions. Therefore, our next result applies to a very wide range of transition functions.
We say that a non-terminal state $s'$ is \emph{controllable} if there is a state $s$ and actions $a_1, a_2$ such that $\mathbb{P}(\tau(s,a_1) = s') \neq \mathbb{P}(\tau(s,a_2) = s')$, and that $\tau$ is \emph{non-trivial} if it has at least one controllable state.



\begin{theorem}\label{thm:no_robustness_most_tau}
Let $f_{\TransitionDistribution,d,\beta}, g_{\TransitionDistribution,d,\beta} \in \{r_{\tau,d,\beta},n_{\tau,d,\beta},s_{\tau,d,\beta}\}$ be two behavioural models. Let $d_1$ and $d_2$ be any two discount functions, and let $\beta_1, \beta_2 \in (0,\infty)$ be any two temperature parameters. 
Let $\TransitionDistribution$ be any non-trivial episodic transition function. Then unless $d_1(1)/d_1(0) = d_2(1)/d_2(0)$, we have that for any reward $R_1$ there exists a reward $R_2$ such that 
$$
f_{\TransitionDistribution,d_1,\beta_1}(R_1) = f_{\TransitionDistribution,d_1,\beta_1}(R_2),
$$ 
but such that 
$$
g_{\TransitionDistribution,d_2,\beta_2}(R_1) \neq g_{\TransitionDistribution,d_2,\beta_2}(R_2).
$$ 
Moreover, unless $d_1(1)/d_1(0) = \gamma$, we also have that there for any reward $R_1$ exists a reward $R_2$ such that $f_{\TransitionDistribution,d_1,\beta_1}(R_1) = f_{\TransitionDistribution,d_1,\beta_1}(R_2)$, but such that $R_1$ and $R_2$ have different optimal policies under exponential discounting with $\gamma$.
\end{theorem}

Nearly all transition functions are non-trivial, so Theorem~\ref{thm:no_robustness_most_tau} applies very broadly.  
Note that Theorem~\ref{thm:no_robustness_some_tau} makes weaker assumptions about the discount function but stronger assumptions about the transition function, whereas Theorem~\ref{thm:no_robustness_most_tau} makes stronger assumptions about the discount function but weaker assumptions about the transition function.

\section{Discussion and Further Work}

We have analysed partial identifiability in IRL with non-exponential discounting, including (but not limited to) hyperbolic discounting. To this end, we have introduced three types of policies (resolute policies, na\"ive policies, and sophisticated policies) that generalise the standard notion of optimality to non-exponential discount functions, and shown that these policies always exist in any episodic MDP.
We have used these policies to generalise the Boltzmann-rational model to non-exponential discounting in three ways, and analysed the identifiability of the reward function under these models. We have demonstrated that the Boltzmann-na\"ive and the Boltzmann-sophisticated policies let us identify the true reward function up to na\"ive and sophisticated potential shaping, and shown that each of the three models in general is too ambiguous (even in the limit of infinite data) to compute the correct policy for a different form of discounting. We have thus made an important contribution to the study of partial identifiability in IRL, by extending existing results to the setting with non-exponential discounting. This is of particular importance, since hyperbolic discounting is considered to be a good fit to human behaviour.

There are several ways that our work can be extended. 
Improving our understanding of identifiability in IRL is of crucial importance, if we want to use IRL (and similar techniques) as a tool for preference elicitation. 
This analysis should consider behavioural models that are actually realistic. We have considered hyperbolic discounting, since this is widespread in the behavioural sciences, but there are many other ways to make our models more psychologically plausible. For example, it would be interesting to incorporate models of human risk-aversion, such as prospect theory \citep{prospecttheory}. Moreover, our analysis is primarily restricted to \emph{episodic} environments with a bounded horizon --- it would be interesting to generalise it to broader classes of environments. This issue is further discussed in the appendix. It would also be interesting to exactly characterise the partial identifiability of the Boltzmann-resolute model --- this issue is also discussed in the appendix. Finally, it would be interesting to study the case where the discount function is \emph{misspecified}, i.e., where the IRL algorithm assumes that the observed agent discounts using some function $d_1$, but where it in fact discounts using some other function $d_2$ \citep[see, e.g.,][]{skalse2023misspecification, skalse2024quantifyingsensitivityinversereinforcement}.

\bibliography{aaai25}

\newpage
\appendix
\setcounter{proposition}{0}
\setcounter{theorem}{0}
\setcounter{lemma}{0}
\setcounter{corollary}{0}

\section{Using More General Environments}

Our results in this paper make use of episodic MDPs. Moreover, we only focus on episodic MDPs with a \emph{bounded horizon} -- that is, MDPs for which there is a number $H$ such that any policy with probability 1 will enter the terminal state after at most $H$ steps, starting from any state. It is possible to give a more general definition of episodic MDPs, by only requiring that any policy with probability 1 \emph{eventually} enters a terminal state, starting from any state. The first definition requires that there is a known, fixed upper bound on the length of any episode, whereas the latter definition permits episodes to be arbitrarily long (though the probability of increasingly long episodes must necessarily get lower and lower, provided that $\States$ and $\Actions$ are finite). For example, consider an environment in which the agent is able to flip a coin, and where it enters a terminal state as soon as it gets heads. This MDP would not be episodic in the first sense, since the agent can get a row of $H$ tails for any finite $H$. However, the environment may be episodic in the second sense (depending on what the rest of the environment looks like), because the probability of \emph{eventually} getting heads is 1 (by the Borel–Cantelli lemma). This observation prompts the following definitions:

\begin{definition}
    The MDP $\MDPd$ is \emph{bounded episodic} (with \emph{horizon $H$}) if any policy $\pi$ with probability $1$ enters $s_\top$ after at most $H$ steps, starting from any state. We say that it is \emph{unbounded episodic} if any policy $\pi$ with probability $1$ \emph{eventually} enters $s_\top$, starting from any state. 
\end{definition}

Note that any bounded episodic MDP is an unbounded episodic MDP, and that the definition of \enquote{episodic MDP} which we use in the main text of the paper corresponds to bounded episodic MDPs. It would be desirable to extend our results to unbounded episodic MDPs. As we will show, many of our proofs already apply to this more general setting, and in those cases, our proofs are expressed in terms of unbounded episodic MDPs. However, some results only apply to bounded episodic MDPs.

Note that it would not be possible to extend our results to (general) non-episodic MDPs, since the value function $V^\pi$ may be undefined in such environments for certain discount functions (including hyperbolic discounting).

\section{Proofs}

In this appendix, we prove all of our theoretical results.

\subsection{Core Lemmas}

Let $d : \Pi \times \Pi \to \mathbb{R}$ be the function given by $d(\pi_1,\pi_2) = \frac{1}{e^t}$, where $t$ is the length of the shortest trajectory $\xi$ such that $\pi_1(\xi) \neq \pi_2(\xi)$, or $0$ if $\pi_1 = \pi_2$.

\begin{lemma}\label{lemma:metric_space}
$(\Pi, d)$ is a compact metric space.
\end{lemma}
\begin{proof}
We must first show that $d$ is a metric, which requires showing that it satisfies the following:
\begin{enumerate}
    \item Identity: $d(\pi_1,\pi_2) = 0$ if and only if $\pi_1=\pi_2$.
    \item Positivity: $d(\pi_1,\pi_2) \geq 0$.
    \item Symmetry: $d(\pi_1,\pi_2) = d(\pi_2,\pi_1)$.
    \item Triangle Inequality: $d(\pi_1,\pi_3) \leq d(\pi_1,\pi_2) + d(\pi_2,\pi_3)$.
\end{enumerate}
It is straightforward to see that 1-3 hold. For 4, let $t$ be the length of the shortest history $h$ such that $\pi_1(h) \neq \pi_3(h)$. Note that if $d(\pi_1,\pi_3) > d(\pi_1,\pi_2)$ and $d(\pi_1,\pi_3) > d(\pi_2,\pi_3)$, then it must be the case that $\pi_1(h) = \pi_2(h)$ for all $h$ of length $\leq t$, and that $\pi_1(h) = \pi_2(h)$ for all $h$ of length $\leq t$. However, this is a contradiction, since it would imply that $\pi_1(h) = \pi_3(h)$ for all $h$ of length $\leq t$. Thus either $d(\pi_1,\pi_3) \leq d(\pi_1,\pi_2)$ or $d(\pi_1,\pi_3) \leq d(\pi_2,\pi_3)$, which in turn implies that $d(\pi_1,\pi_3) \leq d(\pi_1,\pi_2) + d(\pi_2,\pi_3)$.

Thus $d$ is a metric, which means that $(\Pi, d)$ is a metric space. Next, we will prove that $(\Pi, d)$ is compact, by showing that $(\Pi, d)$ is totally bounded and complete.

To see that $(\Pi, d)$ is totally bounded, let $\epsilon$ be an arbitrary positive real number, and let $t = \ln(1/\epsilon)$, so that $\epsilon = 1/e^t$. Moreover, let $\hat{\Pi}$ be the set of all policies that always take action $a_1$ after time $t$ (but which may behave arbitrarily before time $t$). Now $\hat{\Pi}$ is finite, and for every policy $\pi_1$ there is a policy $\pi_2 \in \hat{\Pi}$ such that $d(\pi_1,\pi_2) \leq \epsilon$ (given by letting $\pi_2(\xi) = \pi_1(\xi)$ for all trajectories $\xi$ with length at most $t$). Thus, for every $\epsilon > 0$, $(\Pi, d)$ has a finite cover. Thus $(\Pi, d)$ is totally bounded.

To see that $(\Pi, d)$ is complete, let $\{\pi_i\}_{i=0}^\infty$ be a Cauchy sequence. This implies that for every $\epsilon > 0$ there is a positive integer $N$ such that for all $n,m \geq N$ we have $d(\pi_n, \pi_m) < \epsilon$. In our case, this means that there, for each time $t$ is a positive integer $N$ such that for all $n,m \geq N$, we have that $\pi_n(\xi) = \pi_m(\xi)$ for all trajectories $\xi$ shorter than $t$ steps. We can thus define a policy $\pi_\infty$ by letting $\pi_\infty(\xi) = \delta$ (where $\delta \in \Delta(\Actions)$) if there is an $N$ such that, for all $n \geq N$, we have that $\pi_n(\xi) = \delta$. Now $\lim_{i \to \infty} \{\pi_i\}_{i=0}^\infty = \pi_\infty$, and $\pi_\infty \in (\Pi, d)$. Thus every Cauchy sequence in $(\Pi, d)$ has a limit that is also in $(\Pi, d)$, and so $(\Pi, d)$ is complete.

Every metric space that is totally bounded and complete is compact. We thus have that $(\Pi, d)$ is compact.
\end{proof}

\begin{lemma}\label{lemma:episodic_n_p}
$\MDPd$ is unbounded episodic if and only if there exists $n \in \mathbb{N}, p \in (0,1]$ such that for any policy $\pi$ and state $s$, if $\pi$ is run from $s$, then after $n$ steps, it will have entered $s_\top$ with probability at least $p$.
\end{lemma}
\begin{proof}
For the first direction, assume that there exists $n \in \mathbb{N}, p \in (0,1]$ such that for any policy $\pi$ and any state $s$, if $\pi$ is run from $s$, then after $n$ steps, it will have entered $s_\top$ with probability at least $p$. Then for any policy $\pi$, if $\pi$ is run for $kn$ steps, then the probability that it has \emph{not} entered a terminal state is at most $p^k$. $\sum_{k=1}^\infty p^k < \infty$, so by the Borel–Cantelli lemma, we have that $\pi$ almost surely eventually will enter $s_\top$. Since $\pi$ was chosen arbitrarily, $\MDPd$ must be unbounded episodic.

For the other direction, let$\MDPd$ be an unbounded episodic MDP. Let $\pi$ and $s$ be selected arbitrarily. Since every policy eventually enters $s_\top$ with probability 1, there must be a trajectory $s, a_0, s_1, \dots$ starting in $s$ and ending in $s_\top$, such that each transition has positive probability under $\pi$ and $\tau$. Moreover, the \emph{shortest} such trajectory can contain no more than $|\States|$ states -- otherwise there must be a loop that occurs with probability 1 when running $\pi$ from $s$. If that were the case, then we could construct a policy that stays on this loop indefinitely, which is impossible if the MDP is unbounded episodic. Since $\pi$ and $s$ were selected arbitrarily, this shows that there is an $n = |\States| \in \mathbb{N}$ such that for any policy $\pi$ and state $s$, if $\pi$ is run from $s$, then after $n$ steps, it will have entered a terminal state with positive probability. It remains to be shown that this probability is bounded below by some positive constant $p$.

Let $q(\pi, s)$ be the probability that $\pi$ will have entered a terminal state after $n$ steps, starting in state $s$. Note that this function is continuous, when viewed as a function from $(\Pi,d)$ to $[0,1]$ (with a fixed $s$). In particular, if $\pi_1(\xi) = \pi_2(\xi)$ for all trajectories $\xi$ of length at most $n$, then $q(\pi_1, s) = q(\pi_2, s)$. Thus, for every $\epsilon > 0$ there is a $\delta = \ln(1/n)$ such that if $d(\pi_1, \pi_2) < \delta$, then $|q(\pi_1, s) - q(\pi_2, s)| = 0 < \epsilon$. Moreover, by Lemma~\ref{lemma:metric_space}, we have that $(\Pi, d)$ is a compact metric space. Thus, by the extreme value theorem, for each $s$ there is a policy $\pi_s \in \Pi$ that minimises $q(\pi, s)$. Moreover, we have already established that for any policy $\pi$ and state $s$, if $\pi$ is run from $s$, then after $n$ steps, it will have entered a terminal state with positive probability. Thus $q(\pi_s, s) > 0$. Since $\States$ is finite, we can now set $p$ to $\min_s (\pi_s, s)$, and thus complete the proof.
\end{proof}



\subsection{Convergent Policy Values}

In this section, we provide the proofs of the claims regarding convergent policy values.

\begin{proposition}\label{prop:episodic}
If $\MDPd$ is unbounded episodic, then we have that $|V^\pi(s)| < \infty$ for all policies $\pi$ and all states $s$. 
\end{proposition}
\begin{proof}
As per Lemma~\ref{lemma:episodic_n_p}, in any unbounded episodic MDP, there is an $n$ and a $p$ such that for any state $s$ and policy $\pi$, we have that $\pi$ after $n$ steps will have entered the terminal state with probability at least $p$. Moreover, since $\States$ and $\Actions$ are finite, we have that $m = \max_{s,a,s'}|R(s,a,s')| \leq \infty$. Since $d(t) \in [0,1]$, this means the discounted reward obtained over any sequence of $n$ steps is at least $-mn$, and at most $mn$. Since the probability of entering a terminal state along any such sequence is at least $p$, we have that
$$
\left|V^\pi(s)\right| \leq \left(\frac{mn}{p}\right),
$$
which is finite.
\end{proof}

\begin{proposition}
If $\langle \States, \Actions, \{s_\top\}, \TransitionDistribution, \InitStateDistribution, R_1, d \rangle$ is not unbounded episodic, and $\sum_{t=0}^\infty d(t) = \infty$, then there is a reward function $R_2$, policy $\pi$, and state $s$, such that $V^\pi(s) = \infty$ in $\langle \States, \Actions, \{s_\top\}, \TransitionDistribution, \InitStateDistribution, R_2, d \rangle$.
\end{proposition}
\begin{proof}
Let $R_2$ be the reward such that $R_2(s,a,s') = 1$ for all $s,a,s'$. Now, since $\langle \States, \Actions, \{s_\top\}, \TransitionDistribution, \InitStateDistribution, R_1, d \rangle$ is not unbounded episodic, there is a policy $\pi$ that, with positive probability, never enters a terminal state. Let this probability be $p$. This means that there must be an initial state $s_0$ such that the probability that $\pi$ never enters a terminal state, conditional on the first state being $s_0$, is at least $p$. This means that $V^\pi(s_0) \geq p \cdot \sum_{t=0}^\infty 1 = \infty$ in the MDP $\langle \States, \Actions, \{s_\top\}, \TransitionDistribution, \InitStateDistribution, R_2, d \rangle$. 
\end{proof}

\subsection{Temporal Consistency}

\begin{proposition}
A discount function $d$ is temporally consistent if and only if $d(t) = \alpha\gamma^t$ for some $\alpha, \gamma \in [0,1]$. 
\end{proposition}

The proof of this proposition is given in \citet{temporal_inconsistency_2} (their Theorem 13). Their terminology is slightly different from ours, but their proof applies to our case with essentially no modification.

\subsection{Correspondence To Optimality}

Here, we will establish the relationship between optimal policies, resolute policies, na\"ive policies, and sophisticated policies, in the case of exponential discounting.

\begin{theorem}
If $\MDP$ is an MDP with exponential discounting, then the following are equivalent:
\begin{enumerate}
    \item[1.] $\pi$ is optimal.
    \item[2.] $\pi$ is resolute.
    \item[3.] $\pi$ is na\"ive.
    \item[4.] $\pi$ is sophisticated.
\end{enumerate}
\end{theorem}
\begin{proof}
First of all, in an exponentially discounted MDP, $\pi_1$ is optimal if for all states $s$ and policies $\pi_2$, we have $V^{\pi_1}(s) \geq V^{\pi_2}(s)$, and $\pi_1$ is resolute if for all states $s$, times $t$, and policies $\pi_2$, we have $V^{\pi_1,t}(s) \geq V^{\pi_2,t}(s)$. Moreover, since exponential discounting is temporally consistent, we have that for all $t$, $V^{\pi_1}(s) \geq V^{\pi_2}(s)$ if and only if $V^{\pi_1,t}(s) \geq V^{\pi_2,t}(s)$. From this it follows that 1 and 2 are equivalent in an exponentially discounted MDP.

Secondly, in an exponentially discounted MDP, we have that a policy $\pi$ is optimal if and only if $\mathrm{supp}(\pi(s)) \subseteq \mathrm{argmax}_a(Q^\star(s,a))$, and $\pi$ is na\"ive if and only if for each state $s$, if $a \in \mathrm{supp}(\pi(s))$, then there is a policy $\pi^\star$ such that $\pi^\star$ maximises $V^{\pi^\star}(s)$ and $a \in \mathrm{supp}(\pi^\star(s))$. Moreover, if $\pi^\star$ maximises $V^{\pi^\star}(s)$, then each $a \in \mathrm{supp}(\pi^\star(s))$ must maximise $Q^\star$. From this, it follows that 1 and 3 are equivalent in exponentially discounted MDPs.

Furthermore, in an exponentially discounted MDP, we have that a policy $\pi$ is optimal if and only if it is a fixed point under \emph{policy iteration}, and $\pi$ is sophisticated if and only if $\mathrm{supp}(\pi(s)) \subseteq \mathrm{argmax}Q^\pi(s,a)$. From this, it follows that 1 and 4 are equivalent in exponentially discounted MDPs.
\end{proof}

\subsection{Resolute Policies}

We here provide our proofs about resolute policies.

\begin{lemma}\label{lemma:resolute_value_function_existence}
In any unbounded episodic MDP $\MDPd$, each state $s$ and time $t$, there exists a policy $\pi_1$ such that $V^{\pi_1,t}(s) \geq V^{\pi_2,t}(s)$ for all $\pi_2$.
\end{lemma}
\begin{proof}
We will show that $V^{\pi,t}(s)$ is continuous, when viewed as a function from $(\Pi, d)$ to $\mathbb{R}$. Let $\pi_1$ be any policy, and $\epsilon$ any positive real number.
Since $\States$ and $\Actions$ are finite, we have $m = \max_{s,a,s'} |R(s,a,s')| < \infty$. Moreover, as per Lemma~\ref{lemma:episodic_n_p}, since the MDP is unbounded episodic, there is an $n$ and $p$ such that any policy $\pi$ after $n$ steps will have entered a terminal state with probability at least $p$. Thus, if $\pi_1(\xi) = \pi_2(\xi)$ for all trajectories of length $kn$, then the difference in reward between $\pi_1$ and $\pi_2$ can be at most $mn(1-p)^k/p$. For any $k$ that is sufficiently large (and hence for any $d(\pi_1,\pi_2)$ that is sufficiently small), we have that this quantity is below $\epsilon$. Thus, for every $\epsilon$ there is a $\delta$ such that, if $d(\pi_1, \pi_2) < \delta$ then $|V^{\pi_1,t}(s) -  V^{\pi_2,t}(s)| < \epsilon$. This means that $V^\pi(s,t)$ is continuous, when viewed as a function from $(\Pi, d)$ to $\mathbb{R}$.

By Lemma~\ref{lemma:metric_space}, we have that $(\Pi, d)$ is compact. Thus, by the extreme value theorem, there must exist a policy $\pi_1$ such that $V^{\pi_1}(s,t) \geq V^{\pi_2}(s,t)$ for all $\pi_2$.
\end{proof}

\begin{proposition}
In any unbounded episodic MDP, the resolute $Q$- function $Q^\mathrm{R}$ exists and is unique.
\end{proposition}
\begin{proof}
Immediate from Lemma~\ref{lemma:resolute_value_function_existence}.
\end{proof}

\begin{theorem}
In any unbounded episodic MDP, there exists a deterministic resolute policy.
\end{theorem}
\begin{proof}
By Proposition~\ref{prop:resolute_Q}, in any unbounded episodic MDP, the resolute $Q$- function $Q^\mathrm{R}$ exists and is unique. We now have that any policy $\pi$ is resolute if, for each trajectory $\xi$, we have that $\pi(\xi) \in \mathrm{argmax}_a Q^\mathrm{R}(s, |\xi|, a)$, where $s$ is the last state in $\xi$. There always exists a deterministic policy satisfying this criterion.
\end{proof}

\begin{example}\label{example:delay_MDP}
Let \texttt{Delay} be the bounded episodic MDP where $\States = \{s_0, s_1, s_2, s_3, s_4\}$ $\Actions = \{a_1, s_2\}$, $\mu_0$ is uniform over $s_0$ and $s_2$, and $\TransitionDistribution$ is the deterministic function where $\TransitionDistribution(s_2,a_1) = s_3$, $\TransitionDistribution(s_2,a_2) = s_4$. For all states $s \neq s_2$, we have that $\TransitionDistribution(s,a_1) = \TransitionDistribution(s,a_2)$, where $\TransitionDistribution(s_0,a) = s_1$, $\TransitionDistribution(s_1,a) = s_2$, $\TransitionDistribution(s_3,a) = s_\top$, and $\TransitionDistribution(s_4,a) = s_\top$.
This is depicted in the following labelled graph:
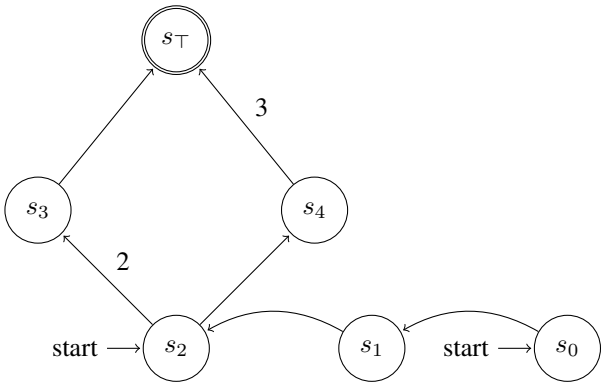
\begin{center}
\begin{tikzpicture}[shorten >=1pt,node distance=2.6cm,on grid,auto]
   \node[state, initial] (s_0)   {$s_0$}; 
   \node[state]         (s_1) [left=of s_0] {$s_1$};
   \node[state,initial]         (s_2) [left=of s_1] {$s_2$};
   \node[state]         (s_3) [above left=of s_2] {$s_3$};
   \node[state]         (s_4) [above right=of s_2] {$s_4$};
   \node[state, accepting]  (s_t) [above=of s_2, yshift=1.5cm] {$s_\top$};
    \path[->] 
    (s_0) edge [bend right] node {} (s_1)
    (s_1) edge [bend right] node {} (s_2)
    (s_2) edge [swap] node {2} (s_3)
    (s_2) edge [] node {} (s_4)
    (s_3) edge [] node {} (s_t)
    (s_4) edge [swap] node {3} (s_t)
    ;
\end{tikzpicture}
\end{center}
The discount function $d$ is the hyperbolic discount function, $d(t) = 1/(1+t)$, and $R$ is the reward function given by  $R(s_2,a_1,s_3) = 2$, $R(s_4,a_1,s_\top) = R(s_4,a_2,s_\top) = 3$, and $R(s,a,s') = 0$ for all other $s,a,s'$.
\end{example}

\begin{proposition}
There are bounded episodic MDPs with no stationary resolute policies.
\end{proposition}
\begin{proof}
Consider the MDP \texttt{Delay} given in Example~\ref{example:delay_MDP}. Here any resolute policy $\pi$ has the property that $\pi(s_2) = a_1$, but $\pi(s_0,a,s_1,a',s_2) = s_2$ (where $a$ and $a'$ may be any actions). Thus this MDP has no stationary resolute policies. 
\end{proof}

\subsection{Na\"ive Policies}

We here provide our proofs about na\"ive policies.

\begin{proposition}
In any unbounded episodic MDP, the na\"ive $Q$-function $Q^\mathrm{N}$ exists and is unique.
\end{proposition}
\begin{proof}
Immediate from Proposition~\ref{prop:resolute_Q}.
\end{proof}

\begin{theorem}
In any unbounded episodic MDP, there exists a stationary deterministic na\"ive policy.
\end{theorem}
\begin{proof}
By Proposition~\ref{prop:resolute_Q}, in any unbounded episodic MDP, the na\"ive $Q$- function $Q^\mathrm{N}$ exists and is unique. We now have that any policy $\pi$ is na\"ive if, for each trajectory $\xi$, we have that $\pi(\xi) \in \mathrm{argmax}_a Q^\mathrm{N}(s, a)$, where $s$ is the last state in $\xi$. There always exists a stationary deterministic policy satisfying this criterion.
\end{proof}

\subsection{Sophisticated Policies}

We here provide our proofs about sophisticated policies. We first show that such policies are guaranteed to exist for bounded episodic MDPs.

\begin{theorem}
In any bounded episodic MDP, there exists a stationary, deterministic sophisticated policy.
\end{theorem}
\begin{proof}
Let $\MDPd$ be a bounded episodic MDP. We will prove that this MDP has a stationary sophisticated policy by induction on the state-space.

Let $S_1, \dots, S_H$ be a partition of $\States$, such that $s \in S_i$ if the longest path that is possible from $s$ has length $i$. Since the MDP is bounded episodic with horizon $H$, this means that some state is in $S_H$. Also note that if $s \in S_i$, and there is an $a$ such that $s' \in \mathrm{supp}(\TransitionDistribution(s,a))$, then $s' \in S_j$ for some $j < i$. Moreover, there must also be some $s'$ and $a$ such that $s' \in \mathrm{supp}(\TransitionDistribution(s,a))$, and such that $s' \in S_{i-1}$. This means that each $S_1, \dots, S_H$ is nonempty, and that if the agent is in a state $s \in S_i$ at time $t$, then it must be in a state $s' \in S_j$ for some $j < i$ at time $t+1$.

Say that a (stationary) policy $\pi$ is sophisticated relative to $S \subseteq \States$ if $\mathrm{supp}(\pi(s)) \subseteq \mathrm{argmax} Q^{\pi,0}(s,a)$ for all $s \in S$, and let $P(n)$ be the claim that there is a stationary, deterministic policy $\pi_n$ that is sophisticated relative to $S = S_1 \cup \dots \cup S_n$.

For our base case $P(1)$, let $s \in S_1$. Note that if $s \in S_1$, then $\TransitionDistribution(s,a) = s_\top$ for all $a$. This means that $\pi_1$ is sophisticated relative to $S_1$ as long as $\mathrm{supp}(\pi_1(s)) \subseteq \mathrm{argmax}_a R(s,a,s_\top)$ for all $s \in S_1$. There is a deterministic stationary policy satisfying this condition, and so $P(1)$ holds.

For our inductive step, assume that there is a stationary deterministic policy $\pi_n$ that is sophisticated relative to $S_1 \cup \dots \cup S_n$. 
Let $Q^{\pi_n,0}$ be the $Q$-function of $\pi_n$. Note that if $s \in S_{n+1}$, then all states that are reachable from $s$ are in $S_1 \cup \dots \cup S_n$. This means that if $\pi = \pi_n$ for all $s \in S_1 \cup \dots \cup S_n$, then for all $s \in S_{n+1}$ and all $a$, we have that $Q^{\pi_n,0}(s,a) = Q^{\pi,0}(s,a)$. Thus, a policy $\pi$ is sophisticated relative to $S_1 \cup \dots \cup S_{n+1}$ as long as  $\pi = \pi_n$ for all $s \in S_1 \cup \dots \cup S_n$, and as long as $\mathrm{supp}(\pi(s)) \subseteq \mathrm{argmax}_a Q^{\pi_n,0}(s,a)$ for all $s \in S_{n+1}$. There is a stationary deterministic policy $\pi_{n+1}$ that satisfies this condition, and so the inductive step holds.

Thus, by the principle of induction, there is a stationary deterministic policy that is sophisticated relative to $\States$. This completes the proof.
\end{proof}

This result can be generalised to arbitrary unbounded episodic MDPs, though this requires a bit more work:

\begin{theorem}
In any unbounded episodic MDP, there exists a stationary sophisticated policy.
\end{theorem}
\begin{proof}
By the Kakutani fixed-point theorem, if $X$ is a non-empty, convex, and compact subset of a Euclidean space $\mathbb{R}^n$, and $\phi : X \to \mathcal{P}(X)$ is a set valued function with the property that
\begin{enumerate}
    \item $\phi(x)$ is non-empty, closed, and convex for all $x \in X$, and
    \item $\phi$ is upper hemicontinuous,
\end{enumerate}
then $\phi$ has a fixed point.

Let $\hat{\Pi}$ be the set of all stationary policies. We say that a policy $\pi_2$ is a \emph{local improvement} of $\pi_1$ in $s$ if $\mathrm{supp}(\pi_2(s)) \subseteq \mathrm{argmax}_{a} Q^{\pi_1}(s,a)$. Let $\phi : \hat{\Pi} \to \mathcal{P}(\hat{\Pi})$ be the function that, given $\pi$, returns the set of all policies which are local improvements of $\pi$ in all $s$. 

We can begin by noting that $\hat{\Pi}$ of course is a non-empty, convex, and compact subset of the Euclidean space $\mathbb{R}^{|S||A|}$. It is immediate from the definition that $\phi$ is both convex and closed. Moreover, since the MDP is unbounded episodic, we have that $Q^{\pi}(s,a)$ exists (i.e.\ is finite) for all $\pi, s, a$, by Proposition~\ref{prop:episodic}. Since there is a finite number of actions, we thus also have that $\phi(\pi)$ is non-empty.

Claude Berge's Maximum Theorem says that if $X$ and $Y$ are topological spaces, and $f : X \times Y \to \mathbb{R}$ is continuous, and if moreover
\begin{enumerate}
    \item $f^\star(y) = \mathrm{sup}\{f(x,y) : x \in X\}$
    \item $C(y) = \{x : f(x,y) = f^\star(x)\}$
\end{enumerate}
then $f^\star$ is continuous, and $C$ is upper hemicontinuous. Let $X$ and $Y$ both be equal to $\Pi$, and let $f : \Pi \times \Pi \to \mathbb{R}$ be the function where $f(\pi_1, \pi_2) = \sum_{s} \mathbb{E}_{a \sim \pi_2(s)}[Q^{\pi_1}(s,a)]$. Now $f$ is continuous, and $C(\pi_1) = \{\pi_2 : f(\pi_1,\pi_2) = f^\star(\pi_1)\} = \phi(\pi_1)$. Claude Berge's Maximum Theorem then implies that $\phi$ is upper hemicontinuous.

The Kakutani fixed-point theorem then implies that $\phi$ must have a fixed point, which means that there must be a sophisticated policy. Moreover, by construction, this policy is stationary.
\end{proof}

However, note that for unbounded episodic MDPs, there may not be any deterministic policy that is sophisticated (unlike what is the case for bounded episodic MDPs). We can provide a concrete example of such an MDP:

\begin{example}\label{example:loopy_mdp}
Let \texttt{Tempt} be the MDP where $\States$ has 32 states $\{s_0, s_1, \dots s_{31}\}$, $\Actions = \{\text{up},\text{down}\}$, and $\mu_0 = s_0$. For $i \in 2 \dots 30$, we have that $\tau(s_i, a) = s_{i+1}$ for both $a \in \Actions$, and we have that $\tau(s_{31}, a) = s_{31}$ for both $a \in \Actions$. At $s_0$, we have that $\tau(s_0, \text{up}) = s_1$ and $\tau(s_0, \text{down}) = s_2$, and at $s_1$, we have that $\tau(s_1, a)$ for both $a \in \Actions$ returns $s_0$ with probability 0.99, and otherwise returns $s_{31}$. The reward function $R$ is zero everywhere, except that $R(s_{30}, a, s_{31}) = 100$ for both $a \in \Actions$, and $R(s_0, \text{up}, s_1) = 1$. The discount $d$ is the hyperbolic discount function, $d(t) = 1/(1+t)$. This environment is depicted in the following graph:
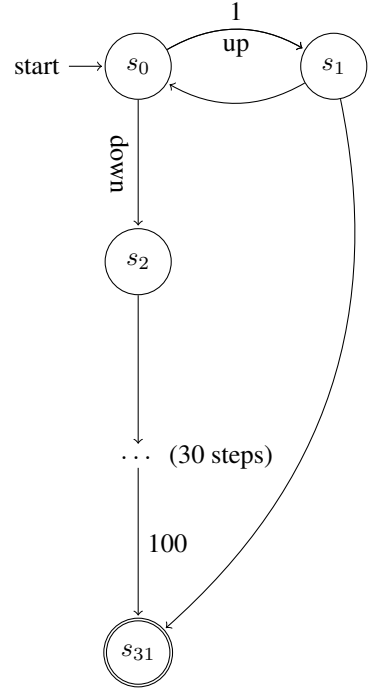
\begin{center}
\begin{tikzpicture}[shorten >=1pt,node distance=2.6cm,on grid,auto]
   \node[state, initial] (s_0)   {$s_0$}; 
   \node[state]         (s_1) [right=of s_0] {$s_1$};
   \node[state]         (s_2) [below=of s_0] {$s_2$};
   \node[]         (invis) [below=of s_2] {$\dots$};
   \node[]         (text) [right=of invis, xshift=-1.5cm] {(30 steps)};
   \node[state, accepting]         (s_31) [below=of invis] {$s_{31}$};
    \path[->] 
    (s_0) edge [sloped, bend left] node {1} (s_1)
          edge [swap, sloped, bend left] node {up} (s_1)
          edge [swap, sloped] node {down} (s_2)
    (s_1) edge [sloped, bend left] node {} (s_0)
          edge [sloped, bend left] node {} (s_31)
    (s_2) edge [sloped] node {} (invis)
    (invis) edge [ ] node {100} (s_31)
    ;
\end{tikzpicture}
\end{center}
Note that \texttt{Tempt} is episodic, with $s_{31}$ being the terminal state. Moreover, state $s_0$ is the only state in which the agent has a meaningful choice to make; in all other states, $\tau$ does not depend on the action. Note also that $\tau$ is deterministic everywhere, except at $s_1$ -- the nondeterminism at $s_1$ is to ensure that \texttt{Tempt} is (unbounded) episodic.
\end{example}

\begin{proposition}
There exists unbounded episodic MDPs in which every sophisticated policy is nondeterministic.
\end{proposition}
\begin{proof}
Consider the MDP \texttt{Tempt}, given in Example~\ref{example:loopy_mdp}, and let $\pi$ be any deterministic policy. There are now two cases; either $\pi$ always selects $\text{up}$, or there exists a $\xi$ such that $\pi(\xi) = \text{down}$.

Case 1: Suppose $\pi(\xi) = \text{up}$ for all $\xi$. We then have
\begin{center}
\begin{tabular}{ l l } 
 $Q^{\pi}(\xi,\text{up}) \approx 3.008 $ & $Q^{\pi}(\xi,\text{down}) = 3.\overline{3}$\rule[-1ex]{0pt}{3ex}
\end{tabular}
\end{center}
We thus have that $Q^{\pi}(\xi,\text{down}) > Q^{\pi}(\xi,\text{up})$, even though $\pi(\xi) = \text{up}$. This means that $\pi$ is not sophisticated.

Case 2: Suppose $\pi(\xi) = \text{down}$ for some $\xi$. We then have
\begin{center}
\begin{tabular}{ l l } 
$Q^{\pi}(\xi,\text{up}) \approx 4.125 $ & $Q^{\pi}(\xi,\text{down}) = 3.\overline{3}$
\end{tabular}
\end{center}
We thus have that $Q^{\pi}(\xi,\text{up}) > Q^{\pi}(\xi,\text{down})$, even though $\pi(\xi) = \text{down}$. This means that $\pi$ is not sophisticated.

Since Case 1 and 2 are exhaustive, this means that no deterministic policy is sophisticated in \texttt{Tempt}. However, \texttt{Tempt} is unbounded episodic, so by Theorem~\ref{thm:soph_existence}, there must be a policy that is sophisticated in \texttt{Tempt}. Hence, every sophisticated policy in \texttt{Tempt} is nondeterministic. 
\end{proof}

Let us now return to bounded episodic MDPs. We will next show that the set of sophisticated policies may not be convex in such MDPs, and that there is no unique sophisticated $Q$-function.

\begin{example}\label{example:crossroads_MDP}
Let \texttt{Crossroads} be the bounded episodic MDP where $\States = \{s_0, s_1, s_2, s_3\}$ $\Actions = \{\text{left},\text{right}\}$, $\mu_0 = s_0$, and $\TransitionDistribution$ is the deterministic function depicted in the following labelled graph:
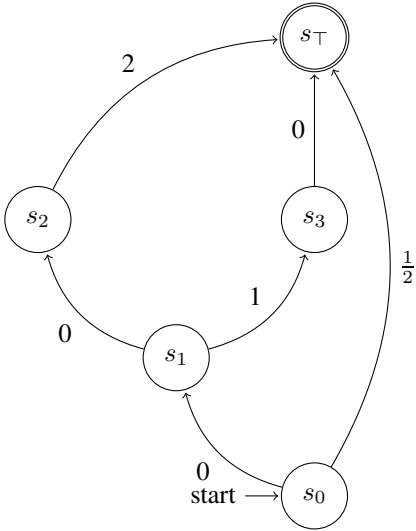
\begin{center}
\begin{tikzpicture}[shorten >=1pt,node distance=2.6cm,on grid,auto]
   \node[state, initial] (s_0)   {$s_0$}; 
   \node[state]         (s_1) [above left=of s_0] {$s_1$};
   \node[state]         (s_2) [above left=of s_1] {$s_2$};
   \node[state]         (s_3) [above right=of s_1] {$s_3$};
   \node[state, accepting]  (s_t) [above=of s_0, yshift=3.5cm] {$s_\top$};
    \path[->] 
    (s_0) edge [bend left] node {0} (s_1)
          edge [swap, bend right] node {$\frac{1}{2}$} (s_t)
    (s_1) edge [bend left] node {0} (s_2)
          edge [bend right] node {1} (s_3)
    (s_2) edge [bend left] node {2} (s_t)
    (s_3) edge [ ] node {0} (s_t)
    ;
\end{tikzpicture}
\end{center}
The discount function $d$ is the hyperbolic discount function, $d(t) = 1/(1+t)$, and $R$ is the reward function given by  $R(s_0,\texttt{right},s_\top) = 1/2$, $R(s_1,a,s_3) = 1$ for all $a$, $R(s_2,a,s_\top) = 2$ for all $a$, and $R(s,a,s') = 0$ for all other $s,a,s'$.
\end{example}

\begin{proposition}
There are bounded episodic MDPs $M$ with hyperbolic discounting and policies $\pi_1$, $\pi_2$ such that both $\pi_1$ and $\pi_2$ are sophisticated in $M$, but such that $Q^{\pi_1} \neq Q^{\pi_2}$, and such that the policy $\pi_3$ given by
$$
\mathbb{P}(\pi_3(s) = a) = \frac{\mathbb{P}(\pi_1(s) = a) + \mathbb{P}(\pi_2(s) = a)}{2}
$$
is not sophisticated in $M$.
\end{proposition}
\begin{proof}
Consider the MDP \texttt{Crossroads} given in Example~\ref{example:crossroads_MDP}. Note that $Q^\pi(s_1,\texttt{left}) = Q^\pi(s_1,\texttt{right})$ for all $\pi$, which means that the agent at $s_1$ is indifferent between both actions. However, $Q^\pi(s_0,\texttt{left}) = p/2 + 2(1-p)/3$, where $p = \mathbb{P}(\pi(s_1) = \texttt{right})$. Thus, at $s_0$, the agent is not indifferent between what it does at $s_1$.

Let $\pi_\texttt{left}$ be the policy that always picks \texttt{left}, and $\pi_\texttt{right}$ be the policy that always picks \texttt{right}. It is now easy to verify that both $\pi_\texttt{left}$ and $\pi_\texttt{right}$ are sophisticated in \texttt{Crossroads}. Moreover, $Q^{\pi_\texttt{left}}(s_0, \texttt{left}) = 2/3$ but $Q^{\pi_\texttt{right}}(s_0, \texttt{left}) = 1/2$, and so $Q^{\pi_\texttt{left}} \neq Q^{\pi_\texttt{right}}$. Additionally, let $\pi_3$ be the policy that averages $\pi_\texttt{left}$ and $\pi_\texttt{right}$. Then $Q^{\pi_3}(s_0,\texttt{left}) = 7/12$ and $Q^{\pi_3}(s_0,\texttt{right}) = 1/2$, but $\texttt{right} \in \mathrm{supp}(\pi_3(s_0))$. Thus $\pi_3$ is not sophisticated.
\end{proof}

\begin{proposition}
In any bounded episodic MDP, the sophisticated $Q$-function $Q^\mathrm{S}$ exists and is unique.
\end{proposition}
\begin{proof}
To prove this, we need to show that any bounded episodic MDP has a canonical sophisticated policy. This proof proceeds by backwards induction, and is analogous to the proof of Theorem~\ref{thm:soph_existence}.

Let $S_1, \dots, S_H$ be a partition of $\States$, such that $s \in S_i$ if the longest path that is possible from $s$ has length $i$. Since the MDP is bounded episodic with horizon $H$, this means that some state is in $S_H$. Also note that if $s \in S_i$, and there is an $a$ such that $s' \in \mathrm{supp}(\TransitionDistribution(s,a))$, then $s' \in S_j$ for some $j < i$. Moreover, there must also be some $s'$ and $a$ such that $s' \in \mathrm{supp}(\TransitionDistribution(s,a))$, and such that $s' \in S_{i-1}$. This means that each $S_1, \dots, S_H$ is nonempty, and that if the agent is in a state $s \in S_i$ at time $t$, then it must be in a state $s' \in S_j$ for some $j < i$ at time $t+1$.

Say that a (stationary) policy $\pi$ is canonical sophisticated relative to $S \subseteq \States$ if $\mathrm{supp}(\pi(s)) \subseteq \mathrm{argmax} Q^{\pi,0}(s,a)$ for all $s \in S$, and $\mathbb{P}(\pi(s) = a_1) = \mathbb{P}(\pi(s) = a_2)$ for all $s \in S$ and all $a_1,a_2 \in \mathrm{argmax} Q^{\pi,0}(s,a)$. Let $P(n)$ be the claim that there is a policy $\pi_n$ that is canonical sophisticated relative to $S = S_1 \cup \dots \cup S_n$.

For our base case $P(1)$, let $s \in S_1$. Note that if $s \in S_1$, then $\TransitionDistribution(s,a) = s_\top$ for all $a$. This means that $\pi_1$ is sophisticated relative to $S_1$ as long as $\mathrm{supp}(\pi_1(s)) \subseteq \mathrm{argmax}_a R(s,a,s_\top)$ for all $s \in S_1$. By mixing uniformly between all such actions we obtain a policy that is canonical sophisticated relative to $S_1$, and so the base case $P(1)$ holds.

For our inductive step, assume that there is a policy $\pi_n$ that is canonical sophisticated relative to $S_1 \cup \dots \cup S_n$.
Let $Q^{\pi_n,0}$ be the $Q$-function of $\pi_n$. Note that if $s \in S_{n+1}$, then all states that are reachable from $s$ are in $S_1 \cup \dots \cup S_n$. This means that if $\pi = \pi_n$ for all $s \in S_1 \cup \dots \cup S_n$, then for all $s \in S_{n+1}$ and all $a$, we have that $Q^{\pi_n,0}(s,a) = Q^{\pi,0}(s,a)$. Thus, a policy $\pi$ is canonical sophisticated relative to $S_1 \cup \dots \cup S_{n+1}$ as long as  $\pi = \pi_n$ for all $s \in S_1 \cup \dots \cup S_n$, and as long as $\pi$ mixes uniformly between all actions in $\mathrm{argmax}_a Q^{\pi_n,0}(s,a)$ for all $s \in S_{n+1}$. There is a policy $\pi_{n+1}$ that satisfies this condition, and so the inductive step holds.

Thus, by the principle of induction, there is a policy $\pi_H$ that is canonical sophisticated relative to $\States$. Now $Q^\mathrm{N} = Q^{\pi_H,0}$, which completes the proof.
\end{proof}



\subsection{Exact Characterisation of Identifiability}

In this section, we provide the proofs of our results that exactly characterise the partial identifiability of the reward for certain behavioural models.

\begin{theorem}
For any bounded episodic MDP $\langle \States, \Actions, \{s_\top\}, \TransitionDistribution, \InitStateDistribution, R_1, d\rangle$, and for any potential function $\Phi : \States \to \mathbb{R}$, there exists a reward function $R_2$ that differs from $R_1$ by na\"ive potential shaping with $\Phi$, and a reward function $R_3$ that differs from $R_1$ by sophisticated potential shaping with $\Phi$.
\end{theorem}
\begin{proof}
Let $M = \langle \States, \Actions, \{s_\top\}, \TransitionDistribution, \InitStateDistribution, R_1, d\rangle$ be any bounded episodic MDP, and let $\Phi$ be any potential function. Using backward induction, we will construct a reward  $R_2$ that differs from $R_1$ by na\"ive potential shaping with $\Phi$, and a reward function $R_3$ that differs from $R_1$ by sophisticated potential shaping with $\Phi$.

Let $S_1, \dots, S_H$ be a partition of $\States$, such that $s \in S_i$ if the longest path that is possible from $s$ has length $i$. Since the MDP is bounded episodic with horizon $H$, this means that some state is in $S_H$. Also note that if $s \in S_i$, and there is an $a$ such that $s' \in \mathrm{supp}(\TransitionDistribution(s,a))$, then $s' \in S_j$ for some $j < i$. Moreover, there must also be some $s'$ and $a$ such that $s' \in \mathrm{supp}(\TransitionDistribution(s,a))$, and such that $s' \in S_{i-1}$. This means that each $S_1, \dots, S_H$ is nonempty, and that if the agent is in a state $s \in S_i$ at time $t$, then it must be in a state $s' \in S_j$ for some $j < i$ at time $t+1$.

Let $P(n)$ be the claim that there exists a reward function $R_{2,n}$ that differs from $R_1$ by na\"ive potential shaping with $\Phi$ for all states $s \in S_1 \cup \dots \cup S_n$, and a reward function $R_{3,n}$ that differs from $R_1$ by sophisticated potential shaping with $\Phi$ for all states $s \in S_1 \cup \dots \cup S_n$.

For our base case $P(1)$, let $s \in S_1$. Note that if $s \in S_1$, then $\TransitionDistribution(s,a) = s_\top$ for all $a$. This means that $Q^\mathrm{N}(s,a) = R(s,a,s_\top)$ and $Q^\mathrm{S}(s,a) = R(s,a,s_\top)$ for all $a$ and all $s \in S_1$. Thus, simply let $R_{2,1}(s,a,s_\top) = R_{3,1}(s,a,s_\top) = R_1(s,a,s_\top) - \Phi(s)$ for all $s \in S_1$, and we can see that the base case $P(1)$ holds.

For our inductive step, assume that there is a reward function $R_{2,n}$ that differs from $R_1$ by na\"ive potential shaping with $\Phi$ for all states $s \in S_1 \cup \dots \cup S_n$. Let $R_{2,n+1} = R_{2,n}$ for all $s \in S_1 \cup \dots \cup S_n$.
Note that if $s \in S_{n+1}$, then all states that are reachable from $s$ are in $S_1 \cup \dots \cup S_n$. This means that for all $s \in S_{n+1}$, we have that $\max_\pi V^{\pi,1}_{2,n+1}(s') = \max_\pi V^{\pi,1}_{2,n}(s')$ for all states $s'$ reachable from $s$. Also recall that
$$
Q^\mathrm{N}(s,a) = \mathbb{E}\left[R(s,a,S') + \max_\pi V^{\pi,1}(S')\right]
$$
where the expectation is over a state $S'$ sampled from $\tau(s,a)$. This rearranges to
$$
\mathbb{E}\left[R(s,a,S')\right] = Q^\mathrm{N}(s,a) - \mathbb{E}\left[\max_\pi V^{\pi,1}(S')\right].
$$
Also recall that we wish for $Q^\mathrm{N}_2(s,a) = Q^\mathrm{N}_1(s,a) - \Phi(s)$. Thus, by setting
$$
R_{2,n+1}(s,a,s') = Q^\mathrm{N}_1(s,a) - \Phi(s) + \mathbb{E}\left[\max_\pi V^{\pi,1}_{2,n}(S')\right]
$$
then we can ensure that $Q^\mathrm{N}_ {2,n+1}(s,a) = Q^\mathrm{N}_1(s,a) - \Phi(s)$ for all $s \in S_{n+1}$. Since $R_{2,n+1} = R_{2,n}$ for all $s \in S_1 \cup \dots \cup S_n$, we also have that this holds for all $s \in S_1 \cup \dots \cup S_n$. The case is analogous for $R_{3,n+1}$ and sophisticated potential shaping. Thus, the inductive step holds. This completes the proof.
\end{proof}

\begin{lemma}\label{lemma:softmax_shift}
Let $v, w \in \mathbb{R}^n$ be two vectors, and let $\beta \in \mathbb{R}^+$. Then
$$
\frac{\exp \beta v_i}{\sum_{j=1}^n \exp \beta v_j} = \frac{\exp \beta w_i}{\sum_{j=1}^n \exp \beta w_j}
$$
for all $i \in \{1 \dots n\}$ if and only if there is a constant scalar $c$ such that $v_i = w_i + c$ for all $i \in \{1 \dots n\}$.
\end{lemma}
\begin{proof}
For the first direction, suppose there is a constant scalar $c$ such that $v_i = w_i + c$ for all $i \in \{1 \dots n\}$. Then
\begin{align*}
    \frac{\exp \beta v_i}{\sum_{j=1}^n \exp \beta v_j} &= \frac{\exp \beta (w_i + c)}{\sum_{j=1}^n \exp \beta (w_j + c)}\\
    &= \frac{\exp (\beta c) \cdot \exp \beta w_i}{\exp (\beta c) \cdot \sum_{j=1}^n \exp \beta w_j}\\
    &= \frac{\exp \beta w_i}{\sum_{j=1}^n \exp \beta w_j}.
\end{align*}
For the other direction, suppose
$$
\frac{\exp \beta v_i}{\sum_{j=1}^n \exp \beta v_j} = \frac{\exp \beta w_i}{\sum_{j=1}^n \exp \beta w_j}
$$
for all $i \in \{1 \dots n\}$. Note that this can be rewritten as follows:
\begin{align*}
    \frac{\exp \beta v_i}{\exp \beta w_i} &= \frac{\sum_{j=1}^n \exp \beta v_j}{\sum_{j=1}^n \exp \beta w_j}\\
    \exp (\beta v_i - \beta w_i) &= \frac{\sum_{j=1}^n \exp \beta v_j}{\sum_{j=1}^n \exp \beta w_j}\\
    v_i - w_i &= \left(\frac{1}{\beta}\right) \log \left(\frac{\sum_{j=1}^n \exp \beta v_j}{\sum_{j=1}^n \exp \beta w_j}\right)
\end{align*}
Since the right-hand side of this expression does not depend on $i$, it follows that $v_i - w_i$ is constant for all $i$.
\end{proof}

\begin{theorem}
In any bounded episodic MDP, the Boltzmann-sophisticated policy determines $R$ up to sophisticated potential shaping.
\end{theorem}
\begin{proof}
Let $\tau$ be a bounded episodic transition function. Recall that $s_{\tau,d,\beta}$ is computed by applying a softmax function (with temperature $\beta$) to $Q^\mathrm{S}$. Moreover, as per Lemma~\ref{lemma:softmax_shift}, any softmax function is invariant to constant shift, and no other transformations. This means that $s_{\tau,d,\beta}(R_1) = s_{\tau,d,\beta})(R_2)$ if and only if there for each state $s$ is a value $\Phi(s)$ such that $Q^\mathrm{S}_2(s,a) = Q^\mathrm{S}_1(s,a) - \Phi(s)$ for all $a$. This is in turn by definition equivalent to $R_1$ and $R_2$ differing by sophisticated potential shaping. As such, $s_{\tau,d,\beta}$ determines $R$ up to sophisticated potential shaping.
\end{proof}

\begin{theorem}
In any bounded episodic MDP, the Boltzmann-na\"ive policy determines $R$ up to na\"ive potential shaping.
\end{theorem}
\begin{proof}
Let $\tau$ be a bounded episodic transition function. Recall that $n_{\tau,d,\beta}$ is computed by applying a softmax function (with temperature $\beta$) to $Q^\mathrm{S}$. Moreover, as per Lemma~\ref{lemma:softmax_shift}, any softmax function is invariant to constant shift, and no other transformations. This means that $n_{\tau,d,\beta}(R_1) = n_{\tau,d,\beta})(R_2)$ if and only if there for each state $s$ is a value $\Phi(s)$ such that $Q^\mathrm{N}_2(s,a) = Q^\mathrm{N}_1(s,a) - \Phi(s)$ for all $a$. This is in turn by definition equivalent to $R_1$ and $R_2$ differing by na\"ive potential shaping. As such, $n_{\tau,d,\beta}$ determines $R$ up to na\"ive potential shaping.
\end{proof}

\subsection{Comparative Characterisation of Identifiability}

\begin{theorem}
Let $f_{\TransitionDistribution,d,\beta}, g_{\TransitionDistribution,d,\beta} \in \{r_{\tau,d,\beta},n_{\tau,d,\beta},s_{\tau,d,\beta}\}$ be two behavioural models. Let $d_1$ and $d_2$ be any two discount functions, and let $\beta_1, \beta_2 \in (0,\infty)$ be any two temperature parameters. Then unless $d_1(t) = d_2(t)$ for all $t \leq |\States|$, there exists an episodic transition function $\TransitionDistribution$ such that for any reward $R_1$ there exists a reward $R_2$ such that 
$$
f_{\TransitionDistribution,d_1,\beta_1}(R_1) = f_{\TransitionDistribution,d_1,\beta_1}(R_2),
$$ 
but such that 
$$
g_{\TransitionDistribution,d_2,\beta_2}(R_1) \neq g_{\TransitionDistribution,d_2,\beta_2}(R_2).
$$ 
Moreover, unless $d(t) = \gamma^t$ for some $\gamma \in [0,1]$ and all $t \leq |\States| - 1$, we also have that, for any reward $R_1$ there exists a reward $R_2$, such that $f_{\TransitionDistribution,d_1,\beta_1}(R_1) = f_{\TransitionDistribution,d_1,\beta_1}(R_2)$, but such that $R_1$ and $R_2$ have different optimal policies under exponential discounting with $\gamma$.
\end{theorem}
\begin{proof}
First assign an integer value to every state in $\States$, so that $\States = \{s_0 \dots s_n\}$, where $s_0 \in \mathrm{supp}(\mu_0)$. We assume that $\Actions$ contains at least two actions $a_1, a_2$. Now consider the transition function $\tau$ where $\tau(s_0, a_1) = s_1$ and $\tau(s_0, a_i) = s_\top$ for all $a_i \neq a_1$. For $i \in \{1 \dots n-1\}$, let $\tau(s_i, a) = s_{i+1}$ for all $a$, and let $\tau(s_n,a) = s_\top$ for all $a$. This function can be visualised as: 
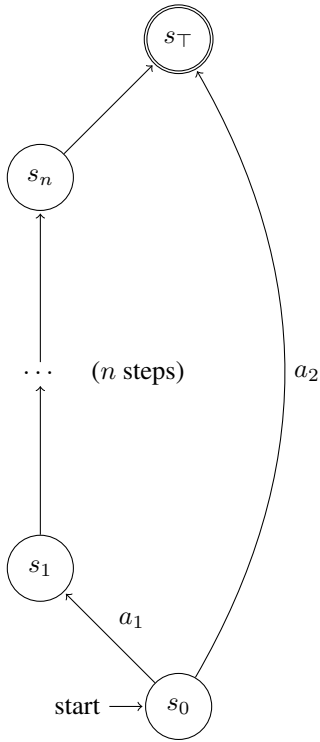
\begin{center}
\begin{tikzpicture}[shorten >=1pt,node distance=2.6cm,on grid,auto]
   \node[state, initial] (s_0)   {$s_0$}; 
   \node[state]         (s_1) [above left=of s_0] {$s_1$};
   \node[]         (invis) [above=of s_1] {$\dots$};
   \node[]         (text) [right=of invis, xshift=-1.3cm] {($n$ steps)};
   \node[state]         (s_minus) [above=of invis] {$s_{n}$};
   \node[state, accepting]         (s_n) [above right=of s_minus] {$s_\top$};
    \path[->] 
    (s_0) edge [swap] node {$a_1$} (s_1)
          edge [swap, bend right] node {$a_2$} (s_n)
    (s_1) edge [sloped] node {} (invis)
    (invis) edge [sloped] node {} (s_minus)
    (s_minus) edge [sloped] node {} (s_n)
    ;
\end{tikzpicture}
\end{center}

By assumption, there is 
a $t \leq |\States|$ such that $d_1(t) \neq d_2(t)$. 
Let $R_1$ be selected arbitrarily, and consider the reward function $R_2$ where $R_2(s_0,a,s_\top) = R_1(s_0,a,s_\top) + x$ for all $a \neq a_1$, $R_2(s_t,a,s_{t+1}) = R_1(s_t,a,s_{t+1}) + x/d_1(t)$ for all $a$, and $R_2 = R_1$ for all other transitions. In other words, if the agent goes right at $s_0$, it will immediately receive an extra $x$ reward, and if it goes left, it will receive an extra $x/d_1(t)$ reward after $t$ steps.

It is now easy to see that if we discount by $d_1$, then $Q_1^\mathrm{R}$ and $Q_2^\mathrm{R}$, $Q_1^\mathrm{N}$ and $Q_2^\mathrm{N}$, and $Q_1^\mathrm{S}$ and $Q_2^\mathrm{S}$ always differ by constant shift in each state (in $s_0$, recall that we assume that $d(0) = 1$ for any discount function). This implies that $f_{\TransitionDistribution,d_1,\beta_1}(R_1) = f_{\TransitionDistribution,d_1,\beta_1}(R_2)$, given that $f_{\TransitionDistribution,d,\beta} \in \{r_{\tau,d,\beta},n_{\tau,d,\beta},s_{\tau,d,\beta}\}$, since the softmax function is invariant to constant shift.

However, when discounting with $d_2$, we have that 
$$
Q_2^\mathrm{X}(s_0, a_1) = Q_1^\mathrm{X}(s_0, a_1) + x \cdot d_2(x)/d_1(x),
$$
but
$$
Q_2^\mathrm{X}(s_0, a_2) = Q_1^\mathrm{X}(s_0, a_2) + x,
$$ 
where $\mathrm{X}$ is either $\mathrm{R}$, $\mathrm{N}$, or $\mathrm{S}$. Since $d_1(t) \neq d_2(t)$, we have that $d_2(x)/d_1(x) \neq 1$. Thus $Q_1^\mathrm{X}$ and $Q_2^\mathrm{X}$ do not differ by constant shift in $s_0$ when $x \neq 0$. This means that $g_{\TransitionDistribution,d_2,\beta_1}(R_1) \neq g_{\TransitionDistribution,d_2,\beta_1}(R_2)$, since the softmax function gives different outputs for any two vectors that do not differ by constant shift.

We can also ensure that $R_1$ and $R_2$ have different optimal policies under exponential discounting with $\gamma$, by picking a sufficiently large or sufficiently small value of $x$. For example, if the policy that is optimal for $R_1$ (under exponential discounting with $\gamma$) takes action $a_1$ at $s_0$, and $d_1(t) > \gamma^t$, then we can ensure that the optimal policy for $R_2$ (under exponential discounting with $\gamma$) instead takes action $a_2$, by picking an $x$ that is sufficiently large. If $d_1(t) < \gamma^t$, then we must instead pick an $x$ that is sufficiently negative. If instead the policy that is optimal for $R_1$ (under exponential discounting with $\gamma$) takes action $a_2$ at $s_0$ and $d_1(t) > \gamma^t$, then we must pick an $x$ that is sufficiently small, and if $d_1(t) < \gamma^t$, then we must pick an $x$ that is sufficiently small.
\end{proof}

\begin{theorem}
Let $f_{\TransitionDistribution,d,\beta}, g_{\TransitionDistribution,d,\beta} \in \{r_{\tau,d,\beta},n_{\tau,d,\beta},s_{\tau,d,\beta}\}$ be two behavioural models. Let $d_1$ and $d_2$ be any two discount functions, and let $\beta_1, \beta_2 \in (0,\infty)$ be any two temperature parameters. 
Let $\TransitionDistribution$ be any non-trivial bounded episodic transition function. Then unless $d_1(1) = d_2(1)$, we have that for any reward $R_1$ there exists a reward $R_2$ such that 
$$
f_{\TransitionDistribution,d_1,\beta_1}(R_1) = f_{\TransitionDistribution,d_1,\beta_1}(R_2),
$$ 
but such that 
$$
g_{\TransitionDistribution,d_2,\beta_2}(R_1) \neq g_{\TransitionDistribution,d_2,\beta_2}(R_2).
$$ 
Moreover, unless $d_1(1) = \gamma$, we also have that there for any reward $R_1$ exists a reward $R_2$ such that $f_{\TransitionDistribution,d_1,\beta_1}(R_1) = f_{\TransitionDistribution,d_1,\beta_1}(R_2)$, but such that $R_1$ and $R_2$ have different optimal policies under exponential discounting with $\gamma$.
\end{theorem}
\begin{proof}
Let $\tau$ be an arbitrary non-trivial bounded episodic transition function, and let $d_1, d_2$ be two arbitrary discount functions such that $d_1(1) \neq d_2(1)$. Moreover, let $R_1$ be an arbitrary reward function. 


Recall that a state $s'$ is \emph{controllable} if there is a non-terminal state $s$ and actions $a_1, a_2$ such that $\mathbb{P}(\tau(s,a_1) = s') \neq \mathbb{P}(\tau(s,a_2) = s')$. Since $\tau$ is non-trivial, there is at least one controllable state. Moreover, since $\tau$ is bounded episodic, there is no state which is reachable from itself. Since $\States$ is finite, there must therefore be a controllable state that cannot be reached from any other controllable state. Call this state $s_c$. Since $s_c$ is not terminal, there are states which are reachable from $s_c$.

Now let $R_2$ be the reward function where $R_2(s,a,s_c) = R_1(s,a,s_c) + x$ and $R_2(s_c,a,s) = R_1(s_c,a,s) - x/d(1)$ for all $s$ and $a$, and $R_2 = R_1$ for all other transitions. We now have that $Q^\mathrm{X}_1$ and $Q^\mathrm{X}_2$ differ by constant shift in each state (and for each time), where $\mathrm{X}$ is either $\mathrm{R}$, $\mathrm{N}$, or $\mathrm{S}$. To see this, note that:


\begin{enumerate}
    \item In all states $s$ which are neither reachable from $s_c$, nor able to reach $s_c$, we of course have that $Q^{\mathrm{X}}_1 = Q^{\mathrm{X}}_2$, for $\mathrm{X} \in \{\mathrm{R},\mathrm{N},\mathrm{S}\}$. $R_1$ and $R_2$ only differ on transitions that begin or end in $s_c$, and so they must induce identical $Q$-functions in states which are disconnected from $s_c$.
    \item In all states $s$ which are reachable from $s_c$, we also have that $Q^{\mathrm{X}}_1 = Q^{\mathrm{X}}_2$ for $\mathrm{X} \in \{\mathrm{R},\mathrm{N},\mathrm{S}\}$. Again, $R_1$ and $R_2$ only differ on transitions that begin or end in $s_c$. Since $\tau$ is bounded episodic (and hence acyclic), we have that if a state $s$ is reachable from $s_c$, then it cannot reach $s_c$. Thus $R_1$ and $R_2$ must induce the same $Q$-functions in such states.
    \item In $s_c$, we have that every outgoing transition gets an extra $x / d(1)$ reward, and that any subsequent transition after that is unchanged. This straightforwardly means that for all actions $a$, we have that $Q^\mathrm{N}_2(s_c, a) = Q^\mathrm{N}_1(s_c, a) + x /d(1)$, that $Q^\mathrm{S}_2(s_c, a) = Q^\mathrm{S}_1(s_c, a) + x /d(1)$, and that $Q^\mathrm{R}_2(s_c,t,a) = Q^\mathrm{R}_1(s_c,t,a) + x / d(1)$ for all $t$.
    \item Finally, consider a state $s$ which can reach $s_c$. Note that any state which is controllable from $s$ cannot reach $s_c$, since $s_c$ is not reachable from any controllable state. 
    
    Let $p(s,a)$ be the probability of transitioning into a state that can reach $s_c$, conditional on taking action $a$ in state $s$. Similarly, let $p(s,a)$ be the probability of transitioning into a state that can not reach $s_c$, conditional on taking action $a$ in state $s$. Finally, let $r(s,a)$ be the probability of transitioning into state $s_c$ conditional on taking action $a$ in state $s$. Let us first consider the na\"ive $Q$-function, $Q^\mathrm{N}$. We can write
    \begin{align*}
    Q^\mathrm{N}_2(s,a) - Q^\mathrm{N}_1(s,a) = &p(s,a) \cdot W_1(s,a) + \\ &q(s,a) \cdot W_2(s,a) + \\ &r(s,a) \cdot W_3(s,a). 
    \end{align*}
    Here $W_1(s,a)$ is 
    \begin{align*}
    \mathbb{E}[&R_2(s,a,S') + \max_\pi V^{\pi,1}_2(S') \\ &- R_1(s,a,S') - \max_\pi V^{\pi,1}_1(S')],    
    \end{align*}
    where the expectation is over a state $S'$ sampled from $\TransitionDistribution(s,a)$, conditional on ending up in a state that can reach $s_c$. Similarly, $W_2(s,a)$ is 
    \begin{align*}
    \mathbb{E}[&R_2(s,a,S') + \max_\pi V^{\pi,1}_2(S') \\ &- R_1(s,a,S') - \max_\pi V^{\pi,1}_1(S')],    
    \end{align*}
    where the expectation is over a state $S'$ sampled from $\TransitionDistribution(s,a)$, conditional on ending up in a state that cannot reach $s_c$, and $W_3(s,a)$ is 
    \begin{align*}
        &R_2(s,a,s_c) + \max_\pi V^{\pi,1}_2(s_c) \\ &- R_1(s,a,s_c) - \max_\pi V^{\pi,1}_1(s_c).
    \end{align*}
    First, note that $R_2(s,a,s') = R_1(s,a,s')$ unless $s' = s_c$. This means that the $R_1$ and $R_2$-terms cancel out for $W_1$ and $W_2$ (but not $W_3$).
    
    Next, note that $p(s,a_1) = p(s,a_2)$ for all $a_1, a_2$, since any state which is controllable from $s$ cannot reach $r_c$. This also means that $W_1(s,a_1) = W_1(s,a_2)$ for all $a_1, a_2$, and so $p(s,a_1) \cdot W_1(s,a_1) = p(s,a_2) \cdot W_1(s,a_2)$ for all $a_1, a_2$. Let this quantity be denoted by $P(s)$.

    Next, note that $W_2(s,a) = 0$, since $V^{\pi,0}_2(s') = V^{\pi,0}_1(s')$ for all states $s'$ that cannot reach $s_c$. Thus $q(s,a) \cdot W_2(s,a) = 0$.

    Moreover, note that $R_2(s,a,s_c) - R_1(s,a,s_c) = x$, and that $\max_\pi V^{\pi,1}_2(s_c) - \max_\pi V^{\pi,1}_1(s_c) = d(1) \cdot x / d(1) = -x$. Thus $W_3(s,a) = 0$.

    Together, this means that $Q^\mathrm{N}_2(s,a) - Q^\mathrm{N}_1(s,a) = P(s)$, which for each given state is a constant value across all actions. Thus $Q^\mathrm{N}_1$ and $Q^\mathrm{N}_2$ differ by constant shift in each state. By analogous reasoning, this can be shown to also hold for $Q^\mathrm{S}_1$ and $Q^\mathrm{S}_2$, and $Q^\mathrm{R}_1$ and $Q^\mathrm{R}_2$.
\end{enumerate}

Since the softmax function is invariant to constant shift, we thus have that $f_{\TransitionDistribution,d_1,\beta_1}(R_1) = f_{\TransitionDistribution,d_1,\beta_1}(R_2)$. However, since $d_1(1) \neq d_2(1)$, we have that $Q^\mathrm{X}_1$ and $Q^\mathrm{X}_2$ do not differ by constant shift (under discounting with $d_2$) in the state $s$ from which $s_c$ is controllable, provided that $x \neq 0$. This means that $g_{\TransitionDistribution,d_2,\beta_2}(R_1) \neq g_{\TransitionDistribution,d_2,\beta_2}(R_2)$. Similarly, by making $x$ sufficiently large or sufficiently small, we can also ensure that $R_1$ and $R_2$ have different optimal policies under exponential discounting with $\gamma$, provided that $d_1(1) \neq \gamma$ (by changing which action is optimal in state $s$).
%
\end{proof}

\section{Additional Results}

\begin{theorem}
Assume we have an episodic MDP, let $u(t) = 1$, and let $\pi_1$ and $\pi_2$ be policies such that
$$
\Evaluation_u(\pi_1) > \Evaluation_u(\pi_2).
$$
Then if $h(t) = 1/(1+k \cdot t)$, then there exist an $N \in \mathbb{N}$ such that for all $n \geq N$, if $h^{+n}(t) = h(t+n)$, we have
$$
\Evaluation_{h^{+n}}(\pi_1) > \Evaluation_{h^{+n}}(\pi_2).
$$
Moreover, there is a $\Gamma \in (0,1)$ such that, for all $\gamma \in [\Gamma,1)$, if $e^\gamma(t) = \gamma^t$, then we have that
$$
\Evaluation_{e^\gamma}(\pi_1) > \Evaluation_{e^\gamma}(\pi_2).
$$
\end{theorem}
\begin{proof}
We will prove this by showing that 
$$
\lim_{n \to \infty} (1+kn) \Evaluation_{h^{+n}}(\pi) = \lim_{\gamma \to 1} \Evaluation_{e^{\gamma}}(\pi) = \Evaluation_u(\pi).
$$
From this, it follows that if $\Evaluation_u(\pi_1) > \Evaluation_u(\pi_2)$, then $\Evaluation_{h^{+n}}(\pi_1) > \Evaluation_{h^{+n}}(\pi_2)$ and $\Evaluation_{e^\gamma}(\pi_1) > \Evaluation_{e^\gamma}(\pi_2)$ for all sufficiently large $n$, and all $\gamma$ sufficiently close to $1$. Note that the $(1+kn)$-term is a scaling term included to prevent $\Evaluation_{h^{+n}}(\pi)$ from approaching zero -- the precise purpose of this will be made more clear later.

Recall that if $\lim_{x \to \infty} f_i(x)$ exists, and if $\sum_{i=0}^\infty f_i$ converges uniformly, then
$$
\lim_{x \to \infty} \sum_{i=0}^\infty f_i(x) = \sum_{i=0}^\infty \lim_{x \to \infty} f_i(x).
$$
Recall also that a sequence of functions $\sum_{i=0}^\infty f_i$ converges uniformly if for all $\epsilon$ there is a $J$ such that if $j \geq J$ then $|\sum_{i=0}^j f_i(x) - \sum_{i=0}^J f_i(x)| \leq \epsilon$ for all $x$.

We first apply this to hyperbolical discounting. Let
$$
f_i(n) = \left(\frac{1 + kn}{1 + k(n + i)}\right) \mathbb{E}_\pi \left[ R_i \right].
$$
That is, $f_i(n)$ is the expected reward of $\pi$ at the $i$'th step, discounted as though it were the $(n + i)$'th step using hyperbolic discounting with parameter $k$, and rescaled such that the first step is not discounted (i.e.\ so that it is multiplied by $1$). Now $(1+kn) \Evaluation_{h^{+n}}(\pi) = \sum_{i = 0}^\infty f_i(n)$.

We can begin by noting that $\lim_{n \to \infty} f_i(n)$ exists, and that it is equal to $\mathbb{E}_\pi\left[R_i\right]$. To show that $\sum_{i=0}^\infty f_i$ converges uniformly, recall that Lemma~\ref{lemma:episodic_n_p} says that there exists a $t$ and a $p$ such that for any policy $\pi$ and any state $s$, we have that if $\pi$ is run from $s$, then it will after $t$ steps have entered a terminal state with probability at least $p$. Moreover, since $\States$ and $\Actions$ are finite, we have that $m = \max_{s,a,s'}|R(s,a,s')| < \infty$. This means that $|\mathbb{E}_\pi[R_i]| \leq mp^{\lfloor i/t\rfloor}$, which in turn also means that $|f_i(n)| \leq mp^{\lfloor i/t\rfloor}$, since $(1+kn)/(1+k(n+i)) \in [0,1]$. This implies that for all $\ell$,
$$
\left| \sum_{i=\ell \cdot t}^\infty f_i(n) \right| \leq \frac{mtp^\ell}{1-p}.
$$
By making $\ell$ large enough, this quantity can be made arbitrarily close to $0$. Thus $\sum_{i=0}^\infty f_i$ converges uniformly.
We therefore have that 
\begin{align*}
\lim_{n \to \infty} (1+kn) \Evaluation_{h^{+n}}(\pi) &= \lim_{n \to \infty} \sum_{i=0}^\infty f_i(n)\\
 &= \sum_{i=0}^\infty \lim_{n \to \infty} f_i(n)\\
 &= \sum_{i=0}^\infty \mathbb{E}_\pi\left[R_i\right]\\
 &= \Evaluation_c(\pi)
\end{align*}
Thus, if we have that $\Evaluation_c(\pi_1) > \Evaluation_c(\pi_2)$, then it follows that $\lim_{n \to \infty} (1+kn) \Evaluation_{h^{+n}}(\pi_1) > \lim_{n \to \infty} (1+kn) \Evaluation_{h^{+n}}(\pi_2)$. 
Moreover, we of course have that $\Evaluation_{h^{+n}}(\pi_1) > \Evaluation_{h^{+n}}(\pi_2)$ if and only if $(1+kn)\Evaluation_{h^{+n}}(\pi_1) > (1+kn)\Evaluation_{h^{+n}}(\pi_2)$. 
Thus $\lim_{n \to \infty} \Evaluation_{h^{+n}}(\pi_1) > \lim_{n \to \infty} \Evaluation_{h^{+n}}(\pi_2)$, which in turn means that there exist an $N \in \mathbb{N}$ such that for all $n \geq N$, we have $\Evaluation_{h^{+n}}(\pi_1) > \Evaluation_{h^{+n}}(\pi_2)$. This completes the first part.

For the second part, simply let
$$
f_i(\gamma) = \gamma^i \mathbb{E}_\pi \left[ R_i \right].
$$
That is, $f_i(\gamma)$ is the expected reward of $\pi$ at the $i$'th step, exponentially discounted with discount factor $\gamma$. Now $\Evaluation_{e^\gamma}(\pi) = \sum_{i = 0}^\infty f_i(\gamma)$. We of course have that $\lim_{\gamma \to 1} f_i(\gamma)$ exists, and that it is equal to $\mathbb{E}_\pi\left[R_i\right]$, and we can show that $\sum_{i=0}^\infty f_i$ converges uniformly using the same argument as before. We therefore have that 
\begin{align*}
\lim_{\gamma \to 1} \Evaluation_{e^\gamma}(\pi) &= \lim_{\gamma \to 1} \sum_{i=0}^\infty f_i(\gamma)\\
 &= \sum_{i=0}^\infty \lim_{1 \to \gamma} f_i(\gamma)\\
 &= \sum_{i=0}^\infty \mathbb{E}_\pi\left[R_i\right]\\
 &= \Evaluation_c(\pi)
\end{align*}
Thus, if $\Evaluation_c(\pi_1) > \Evaluation_c(\pi_2)$, then $\lim_{\gamma \to 1} \Evaluation_{e^{\gamma}}(\pi_1) > \lim_{\gamma \to 1} \Evaluation_{e^{\gamma}}(\pi_2)$, which in turn means that there is a $\Gamma \in (0,1)$ such that, for all $\gamma \in [\Gamma,1)$, we have that $\Evaluation_{e^\gamma}(\pi_1) > \Evaluation_{e^\gamma}(\pi_2)$. This completes the second part, and the proof.
\end{proof}

\section{The Ambiguity of Resolute Policies}

It would be desirable to extend Theorem~\ref{thm:naive_ambiguity} and \ref{thm:sophisticated_ambiguity} to also cover the Boltzmann-resolute behavioural model. Note that this result necessarily would be closely analogous to Theorem~\ref{thm:naive_ambiguity} and \ref{thm:sophisticated_ambiguity}. In particular, the Boltzmann-resolute policy is given by applying a softmax function to the resolute $Q$-function, $Q^\mathrm{R}$. Also recall that the softmax function is invariant to constant shift, and no other transformations (Lemma~\ref{lemma:softmax_shift}). This means that $R_1$ and $R_2$ have the same Boltzmann-resolute policy if and only if $Q^\mathrm{R}_1$ and $Q^\mathrm{R}_2$ differ by constant shift in all states (in other words, a kind of \enquote{resolute} potential shaping). The difficulty in completing this result lies in generalising Theorem~\ref{thm:potential_shaping_exists}, by showing when it is the case that there exists a reward function $R_2$ such that $Q^\mathrm{R}_2(s,t,a) = Q^\mathrm{R}_1(s,t,a) + \Phi(s,t)$ for some reward $R_1$ and potential function $\Phi$. For na\"ive and sophisticated potential shaping, we have that this always exists (for any $R_1$, $\Phi$, and $\TransitionDistribution$). However, for \enquote{resolute} potential shaping, this will not necessarily be the case. The reason for this is the fact that $Q^\mathrm{R}$ depends on the time, and not just the current state and action, whereas $R_2$ cannot depend on the time. This means that we cannot use backward induction to find an $R_2$ for an arbitrary $R_1$, $\Phi$, and $\TransitionDistribution$.

\end{document}